\documentclass{article}

\usepackage[utf8]{inputenc}

\usepackage{microtype}
\usepackage{graphicx}
\usepackage{subfigure} 
\usepackage{booktabs} 

 \usepackage{hyperref}

\usepackage[accepted]{icml2020} 

\usepackage{enumitem} 

\usepackage{multirow, makecell} 

\usepackage{amsmath, amssymb, amsthm, mathtools}

\DeclareMathAlphabet{\mathpzc}{OT1}{pzc}{m}{it}

\usepackage{bm, bbm, dsfont}
\newtheorem{theorem}{Theorem}
 
\newtheorem{definition}{Definition} 
\newtheorem{lemma}{Lemma}
\newtheorem{remark}{Remark}

\newtheorem{proposition}{Proposition}
\newtheorem{assumption}{Assumption}

\newtheorem*{theorem*}{Theorem}
\newtheorem*{example*}{Example} 
\newtheorem*{definition*}{Definition}
\newtheorem*{lemma*}{Lemma}
\newtheorem*{remark*}{Remark}
\newtheorem*{corollary*}{Corollary}
\newtheorem*{proposition*}{Proposition}
\newtheorem*{assumption*}{Assumption}
\newtheorem*{claim*}{Claim}

\newtheoremstyle{TheoremNum}
        {\topsep}{\topsep}              
        {\itshape}                      
        {}                              
        {\bfseries}                     
        {.}                             
        { }                             
        {\thmname{#1}\thmnote{ \bfseries #3}}
\theoremstyle{TheoremNum}
\newtheorem{theoremn}{Theorem}

\newtheoremstyle{LemmaNum}
        {\topsep}{\topsep}              
        {\itshape}                      
        {}                              
        {\bfseries}                     
        {.}                             
        { }                             
        {\thmname{#1}\thmnote{ \bfseries #3}}
\theoremstyle{LemmaNum}
\newtheorem{lemman}{Lemma}

\newtheoremstyle{PropositionNum}
        {\topsep}{\topsep}              
        {\itshape}                      
        {}                              
        {\bfseries}                     
        {.}                             
        { }                             
        {\thmname{#1}\thmnote{ \bfseries #3}}
\theoremstyle{PropositionNum}
\newtheorem{propositionn}{Proposition}

\newtheoremstyle{RemarkNum}
        {\topsep}{\topsep}              
        {\itshape}                      
        {}                              
        {\bfseries}                     
        {.}                             
        { }                             
        {\thmname{#1}\thmnote{ \bfseries #3}}
\theoremstyle{RemarkNum}
\newtheorem{remarkn}{Remark}

\usepackage{tikz}

\newcommand{\PsiBound}{  
      \max \left\{ \log (  T^2 / \epsilon) , 2 \log_2 ( 1/ \eta )  \right\}
}

\newcommand{\BoundHoeffding}[2]{   \frac{ \left( {\Psi}  \hspace{-2pt} + \hspace{-2pt} {D_{\mathpzc{E}} } \right) \hspace{-2pt} \sqrt{ \HoeffTSq }  }{ \sqrt{ \widetilde{n}_{#1} ( #2 )  } }  }

\newcommand{\BoundMeasure}[1]{  \log  \left( {#1} \right)  }

\newcommand{ \maxHitNumber }[1]{ \frac{2 (\Psi \hspace{-.8pt} + \hspace{-.8pt} D_{\mathpzc{E}} )^2 \log (2 T^2 / \epsilon) }{ \alpha^2 \left[ \log \left(  {#1} \right) \right]^2 } } 

\newcommand{\BoundMeasurecube}[1]{   \log ( \mu ( #1 ) / \eta ) }

\newcommand{\AverageValue}[1]{ \left< f \right>_{#1} }

\newcommand{\Q}{ \mathcal{Q} }

\newcommand{\crazysum}{\sum_{ i = 0 }^ {I-1} }

\newcommand{\HoeffTSq}{2 \log (2 T^2/ \epsilon)}

\newcommand{\h}{\hspace*{-1pt}}

\newcommand{\F}{\mathcal{F}}

\newcommand{\pgftextcircled}[1]{
    \setbox0=\hbox{#1}%
    \dimen0\wd0%
    \divide\dimen0 by 2%
    \begin{tikzpicture}[baseline=(a.base)]%
        \useasboundingbox (-\the\dimen0,0pt) rectangle (\the\dimen0,1pt);
        \node[circle,draw,outer sep=0pt,inner sep=0.1ex] (a) {#1};
    \end{tikzpicture}
}

\usepackage{enumitem} 

\icmltitlerunning{Bandits for BMO Functions}

\begin{document}

\twocolumn[
\icmltitle{Bandits for BMO Functions}

\begin{icmlauthorlist}
\icmlauthor{Tianyu Wang}{tianyu}
\icmlauthor{Cynthia Rudin}{tianyu}
\end{icmlauthorlist}

\icmlcorrespondingauthor{Tianyu Wang}{tianyu@cs.duke.edu}

\icmlaffiliation{tianyu}{Department of Computer Science, Duke University, Durham, NC, USA}

\icmlkeywords{Machine Learning, ICML}

\vskip 0.3in
]

\printAffiliationsAndNotice{}

\begin{abstract} 
We study the bandit problem where the underlying expected reward is a Bounded Mean Oscillation (BMO) function. BMO functions are allowed to be discontinuous and unbounded, and are useful in modeling signals with infinities in the domain. 
We develop a toolset for BMO bandits, and provide an algorithm that can achieve poly-log $\delta$-regret -- a regret measured against an arm that is optimal after removing a $\delta$-sized portion of the arm space. 
\end{abstract}

\section{Introduction} \label{introduction}

Multi-Armed Bandit (MAB) problems model sequential decision making under uncertainty. Algorithms for this problem have important real-world applications including medical trials \citep{robbins1952some} and web recommender systems \citep{li2010contextual}. While bandit methods have been developed for various settings, one problem setting that has not been studied, to the best of our knowledge, is when the expected reward function is a Bounded Mean Oscillation (BMO) function in a metric measure space. 
Intuitively, a BMO function does not deviate too much from its mean over any ball, and can be discontinuous or unbounded. 

Such unbounded functions can model many real-world quantities. Consider the situation in which we are optimizing the parameters of a process (e.g., a physical or biological system) whose behavior can be simulated. The simulator is computationally expensive to run, which is why we could not exhaustively search the (continuous) parameter space for the optimal parameters. The ``reward'' of the system is sensitive to parameter values and can increase very quickly as the parameters change. In this case, by failing to model the infinities, even state-of-the-art continuum-armed bandit methods fail to compute valid confidence bounds, potentially leading to underexploration of the important part of the parameter space, and they may completely miss the optima. 

As another example, when we try to determine failure modes of a system or simulation, we might try to locate singularities in the variance of its outputs. These are cases where the variance of outputs becomes extremely large. In this case, we can use a bandit algorithm for BMO functions to efficiently find where the system is most unstable. 

There are several difficulties in handling BMO rewards. First and foremost, due to unboundedness in the \textit{expected} reward functions, traditional regret metrics are doomed to fail. To handle this, we define a new performance measure, called $\delta$-regret. 
The $\delta$-regret measures regret against an arm that is optimal after removing a $\delta$-sized portion of the arm space. 
Under this performance measure, and because the reward is a BMO function, our attention is restricted to a subspace on which the expected reward is finite. 
Subsequently, strategies that conform to the $\delta$-regret are needed. 

To develop a strategy that handles $\delta$-regret, we leverage the John-Nirenberg inequality, which plays a crucial role in harmonic analysis. We construct our arm index using the John-Nirenberg inequality, in addition to a traditional UCB index. In each round, we play an arm with highest index. As we play more and more arms, we focus our attention on regions that contain good arms. To do this, we discretize the arm space adaptively, and carefully control how the index evolves with the discretization. We provide two algorithms -- \textit{Bandit-BMO-P} and \textit{Bandit-BMO-Z}. They discretize the arm space in different ways. In \textit{Bandit-BMO-P}, we keep a strict partitioning of the arm space. In \textit{Bandit-BMO-Z}, we keep a collection of cubes where a subset of cubes form a discretization. \textit{Bandit-BMO-Z} 
achieves poly-log $\delta$-regret with high probability.

\section{Related Works} 
Bandit problems in different settings have been actively studied since as far back as \citet{thompson1933likelihood}. 
\textit{Upper confidence bound} (UCB) algorithms remain popular \citep{robbins1952some, lai1985asymptotically, auer2002using} among the many approaches for (stochastic) bandit problems \citep[e.g., see][]{srinivas2009gaussian, abbasi2011improved, agrawal2012analysis, bubeck2012best, seldin2014one}. Various extensions of upper confidence bound algorithms have been studied. Some works use KL-divergence to construct the confidence bound \citep{lai1985asymptotically, garivier2011kl, maillard2011finite}, and some works include variance estimates within the confidence bound \citep{audibert2009exploration, auer2010ucb}. UCB is also used in the contextual setting \citep[e.g.,][]{li2010contextual, krause2011contextual, slivkins2014contextual}.

Perhaps Lipschitz bandits are closest to BMO bandits. The Lipschitz bandit problem was termed ``continuum-armed bandits'' in early stages \citep{agrawal1995continuum}. In ``continuum-armed bandits,'' arm space is continuous -- e.g., $[0,1]$. Along this line, bandits that are Lipschitz continuous (or H\"older continuous) have been studied. In particular, \citet{kleinberg2005nearly} proves a $\Omega (T^{2/3})$ lower bound and proposes a $\widetilde{ \mathcal{O}} \left( T^{2/3} \right)$ algorithm. Under other extra conditions on top of Lipschitzness, regret rate of $\widetilde{\mathcal{O}} (T^{1/2})$ was achieved \citep{cope2009regret, auer2007improved}. 
For general (doubling) metric spaces, the Zooming bandit algorithm \citep{kleinberg2008multi} and Hierarchical Optimistic Optimization algorithm \citep{bubeck2011x} were developed. 
In more recent years, some attention has been given to Lipschitz bandit problems with certain extra conditions. To name a few, \citet{bubeck2011lipschitz} study Lipschitz bandits for differentiable rewards, which enables algorithms to run without explicitly knowing the Lipschitz constants. The idea of robust mean estimators \citep{bubeck2013bandits, bickel1965some, alon1999space} was applied to the Lipschitz bandit problem to cope with heavy-tail rewards, leading to the development of a near-optimal algorithm \citep{lu2019optimal}. Lipschitz bandits with an unknown metric, where a clustering is used to infer the underlying unknown metric, has been studied by \citet{christina2019nonparametric}. Lipschitz bandits with discontinuous but bounded rewards were studied by \citet{pmlr-v99-krishnamurthy19a}. 
 
An important setting that is beyond the scope of the aforementioned works is when the expected reward is allowed to be unbounded. This setting breaks the previous Lipschitzness assumption or ``almost Lipschitzness'' assumption \citep{pmlr-v99-krishnamurthy19a}, which may allow discontinuities but require boundedness. To the best of our knowledge, this paper is the first work that studies the bandit learning problem for BMO functions.

\section{Preliminaries} \label{sec:preliminaries}
We review the concept of (rectangular) Bounded Mean Oscillation (BMO) in Euclidean space \citep[e.g.,][]{fefferman1979bounded, stein1993harmonic}.

\begin{definition} (BMO Functions)
\label{def:rbmo}
Let $(\mathbb{R}^d, \mu)$ be the Euclidean space with the Lebesgue measure. Let $L^1_{loc} (\mathbb{R}^d, \mu)$ denote the space of measurable functions (on $\mathbb{R}^d$) that are locally integrable with respect to $\mu$. A function $f \in L^1_{loc} (\mathbb{R}^d, \mu)$ is said to be a Bounded Mean Oscillation function, $ f \in BMO( \mathbb{R}^d, \mu)$, if there exists a constant $C_f$, such that for any hyper-rectangles $Q \subset \mathbb{R}^d$, 
\begin{align}
    &\frac{1}{\mu (Q)} \int_Q |f - \left< f \right>_Q| d\mu \le C_f, 
\quad \left< f \right>_Q := \frac{ \int_Q f d\mu }{\mu(Q)} . \label{eq:def<>}
\end{align}
For a given such function $f$, the infimum of the admissible constant $C_f$ over all hyper-rectangles $Q$ is denoted by $\|f\|_{BMO}$, or simply $\|f\|$. We use $\|f\|_{BMO}$ and $\|f\|$ interchangeably in  this paper. 

\end{definition}

A BMO function can be discontinuous and unbounded. 
The function in Figure \ref{fig:delta-regret} illustrates the singularities a BMO function can have over its domain. Our problem is most interesting when multiple singularities of this kind occur. 
%

To properly handle the singularities, we will need the John-Nirenberg inequality (Theorem \ref{thm:john-nirenberg}), which plays a central role in our paper. 


%
\begin{theorem}[John-Nirenberg inequality] \label{thm:john-nirenberg}
    Let $\mu$ be the Lebesgue measure.
    Let $f\in BMO \left(\mathbb{R}^{d}, \mu \right)$. Then there exists constants $C_1$ and $C_2$, such that, for any hypercube $q \subset \mathbb{R}^{d}$ and any $\lambda > 0$, 
\begin{align}
     \mu \hspace{-2pt} \left( \hspace{-1pt} \left\{ \hspace{-1pt} x \hspace{-1pt} \in \hspace{-1pt} q: \hspace{-1pt} \left|f \hspace{-1pt} (x) \hspace{-1pt} - \hspace{-1pt} \left<f \right>_{q}\right| \hspace{-1pt} > \hspace{-1pt} \lambda \hspace{-1pt} \right\} \hspace{-1pt} \right) 
    \hspace{-1pt} \leq \hspace{-1pt} C_1 { \mu ( q ) } \exp \left\{ \hspace{-1pt} \frac{ - \lambda }{ C_2 \|f\| } \hspace{-1pt} \right\} \hspace{-2pt}. \label{eq:john-nirenberg} 
\end{align}
\end{theorem}
The John-Nirenberg inequality dates back to at least \citet{john1961rotation}, and a proof is provided in Appendix \ref{app:jn}. 

As shown in Appendix \ref{app:jn}, $C_1 = e$ and $C_2 = e2^d$ provide a pair of legitimate $C_1, C_2$ values. However, this pair of $C_1$ and $C_2$ values may be overly conservative. Tight values of $C_1$ and $C_2$ are not known in general cases \citep{lerner2013john, slavin2017john}, and it is also conjectured that $C_2$ and $C_1$ might be independent of dimension \citep{cwikel2012new}. For the rest of the paper, we use $\| f \| = 1$, $C_1 = 1$, and $C_2 = 1$, which permits cleaner proofs. Our results generalize to cases where $C_1$, $C_2$ and $\| f \|$ are other constant values. 

In this paper, we will work in Euclidean space with the Lebesgue measure. For our purpose, Euclidean space is as general as doubling spaces, since we can always embed a doubling space into a Euclidean space with some distortion of metric. This fact is formally stated in Theorem \ref{thm:embedding}. 

\begin{theorem}\citep{assouad1983plongements}. 
    \label{thm:embedding}
    Let $(X, d)$ be a doubling metric space and $\varsigma \in (0,1)$. Then $(X, d^\varsigma)$ admits a bi-Lipschitz embedding into $\mathbb{R}^n$ for some $n \in \mathbb{N}$. 
\end{theorem}

In a doubling space, any ball of radius $\rho$ can be covered by $M_d$ balls of radius $\frac{\rho}{2}$, where $M_d$ is the \textbf{\textit{doubling constant}}. In the space $(\mathbb{R}^d$,$ \| \cdot \|_{\infty})$, the doubling constant $M_d$ is $2^d$. In domains of other geometries, the doubling constant can be much smaller than exponential. Throughout the rest of the paper, we use $M_d$ to denote the doubling constant. 


\section{Problem Setting: BMO Bandits} \label{sec:setting}
The goal of a stochastic bandit algorithm is to exploit the current information, and explore the space efficiently. 
In this paper, we focus on the following setting: a payoff function is defined over the arm space $( [0,1)^d, \|\cdot\|_{\max}, \mu)$, where $\mu$ is the Lebesgue measure (note that $[0,1)^d$ is a Lipschitz domain). The payoff function is: 
\begin{align}
    f: [0,1)^d \rightarrow \mathbb{R} \quad \text{where} \quad f \in BMO ([0,1)^d, \mu). \label{eq:baby-def}
\end{align}
The actual observations are given by $y(a) = f(a) + \mathpzc{E}_a$, where $ \mathpzc{E}_a $ is a zero-mean noise random variable whose distribution can change with $a$. We assume that for all $a$, $ | \mathpzc{E}_a | \le D_\mathpzc{E}$ almost surely for some constant $D_\mathpzc{E}$ (\textbf{N1}). Our results generalize to the setting with sub-Gaussian noise \citep{shamir2011variant}. We also assume that the expected reward function does not depend on noise. 


In our setting, an agent is interacting with this environment in the following fashion. At each round $t$, based on past observations $(a_1, y_1, \cdots, a_{t-1}, y_{t-1} )$, the agent makes a query at point $a_t$ and observes the (noisy) payoff $y_t$, where $y_t$ is revealed only after the agent has made a decision $a_t$.
For a payoff function $f$ and an arm sequence $a_1, a_2, \cdots, a_T$, we use $\delta$-regret incurred up to time $T$ as the performance measure (Definition \ref{def:regret}). 


\begin{definition} ($\delta$-regret)  \label{def:regret}
    Let $f \in BMO ( [0,1)^d, \mu )$. 
    A number $\delta \ge 0$ is called $f$-\texttt{admissible} if there exists a real number $z_0$ that satisfies 
    \begin{align}
        \mu ( \{a \in [0,1)^d: f(a) > z_0 \} ) = \delta. 
    \end{align} 
    For an $f$-admissible $\delta$, define the set $ F^\delta $ to be 
    \begin{align} 
        F^\delta :=  \left\{ z \in \mathbb{R} :  \mu ( \{a \in [0,1)^d: f(a) > z \} ) = \delta \right\}. \label{eq:def-F-epsilon} 
    \end{align} 
    Define
	$
        f^\delta :=  \inf F^\delta. 
    $
    For a sequence of arms $A_1, A_2, \cdots$, and $\sigma$-algebras $\mathcal{F}_1, \mathcal{F}_2, \cdots$ where $\mathcal{F}_t$ describes all randomness before arm $A_t$, define the $\delta $-regret at time $t$ as 
    \begin{align}
        r_t^\delta := \max \{ 0, f^\delta - \mathbb{E}_t [  f (A_t) ] \}, \label{eq:def-epsilon-regret}
    \end{align}
    where $\mathbb{E}_t$ is the expectation conditioned on $\mathcal{F}_t$.
    The total $\delta $-regret up to time $T$ is then $R_T^\delta := \sum_{t=1}^T r_t^\delta. $ 
\end{definition} 

Intuitively, the $\delta$-regret is measured against an amended reward function that is created by chopping off a small portion of the arm space where the reward may become unbounded. As an example, Figure \ref{fig:delta-regret} plots a BMO function and its $f^\delta$ value. 
\begin{figure}[ht!]
	\centering
	\includegraphics[width = 0.45\textwidth]{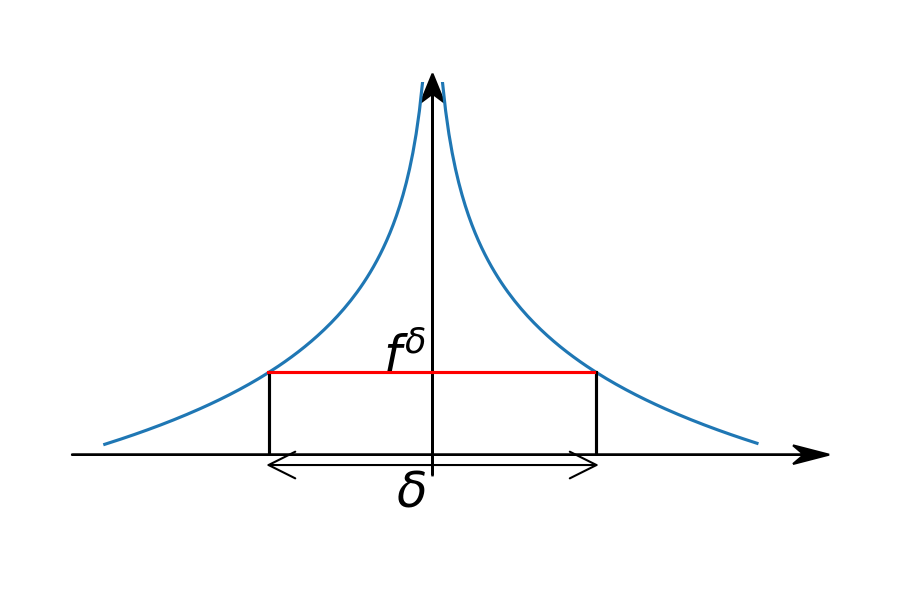}
	\caption{Graph of $ f(x) = - \log (|x|) $, with $\delta$ and $f^\delta$  annotated. This function is an unbounded BMO function. \label{fig:delta-regret}} 
\end{figure}
A problem defined as above with performance measured by $\delta$-regret is called a \textit{\textbf{BMO bandit problem}}. 

\begin{remark}
   	The definition of $\delta$-regret, or a definition of this kind, is needed for a reward function $ f 
    \in BMO([0,1)^d, \mu) $. For an unbounded BMO function $f$, the max value is infinity, while $f^\delta$ is a finite number as long as $\delta$ is $f$-admissible. 
\end{remark} 

\begin{remark}[Connection to bandits with heavy-tails]
    In the definition of bandits with heavy tails \citep{bubeck2013bandits, medina2016no, shao2018almost, lu2019optimal}, the reward distribution at a fixed arm is heavy-tail -- having a bounded expectation and bounded $(1 + \beta)$-moment ($\beta \in (0,1]$). In the case of BMO rewards, the expected reward itself can be unbounded. 
    Figure \ref{fig:delta-regret} gives an instance of unbounded BMO reward, which means the BMO bandit problem is not covered by settings of bandits with heavy tails. 
\end{remark} 

A quick consequence of the definition of $\delta$-regret is the following lemma. This lemma is used in the regret analysis when handling the concentration around good arms. 

\begin{lemma} \label{lem:delta-regret-set-measure}
    Let $f$ be the reward function. For any $f$-admissible $\delta \ge 0$, let $ S^\delta := \left\{ a \in [0,1)^d : f(a) > f^\delta \right\}$. Then we have $S^\delta$ measurable and $\mu(S^\delta) = \delta$. 
\end{lemma}

Before moving on to the algorithms, we put forward the following assumption. 


\begin{assumption} \label{assumption:mean-zero}
	We assume that the expected reward function $f \in BMO ([0,1)^d, \mu)$ satisfies $\left< f \right>_{[0,1 )^d} = 0$. 
\end{assumption}


Assumption \ref{assumption:mean-zero} does not sacrifice generality. 
Since $f$ is a BMO function, it is locally-integrable. Thus $\left< f \right>_{[0,1)^d}$ is finite, and we can translate the reward function up or down such that $\left< f \right>_{[0,1)^d} = 0$. 

\section{Solve BMO Bandits via Partitioning} \label{sec:algo} 
BMO bandit problems can be solved by partitioning the arm space and treating the problem as a finite-arm problem among partitions. For our purpose, we maintain a sequence of partitions using \textit{dyadic cubes}. By \textit{dyadic cubes} of $\mathbb{R}^d$, we refer to the collection of all cubes of the following form: 
\begin{align}
\mathpzc{Q}_{\mathbb{R}^d} :=
     \left\{   \Pi_{i=1}^d \left[ m_i 2^{-k} , m_i 2^{-k} + 2^{-k} \right) \right\} \label{eq:dyadic-cubes}
\end{align}
where $\Pi$ is the Cartesian product, and $m_1, \cdots, m_d, k \in \mathbb{Z}$. Dyadic cubes of $ [0,1)^d $ is $\mathpzc{Q}_{\hspace{1pt}[0,1)^d} \hspace{-2pt} := \hspace{-2pt} \left\{ q \hspace{-2pt} \in \hspace{-2pt} \mathpzc{Q}_{\hspace{1pt}\mathbb{R}^d} : q \hspace{-2pt} \subset [0,1)^d \right\} $. Dyadic cubes of $[0,1)^2$ are 
$
\{ [0,1)^2, [0,0.5)^2, [0.5,1)^2, [0.5,1) \times [0,0.5), \cdots \}. 
$
%

We say a dyadic cube $Q$ is a \textit{\textbf{direct sub-cube}} of a dyadic cube $Q'$ if $Q \subseteq Q'$ and the edge length of $Q'$ is twice the edge length of $Q$. By definition of doubling constant, for any cube $Q$, it has $M_d$ direct sub-cubes, and these direct sub-cubes form a partition of $Q$. If $ Q $ is a direct sub-cube of $Q'$, then $Q'$ is a \textit{\textbf{direct super cube}} of $Q$. 

At each step $t$, Bandit-BMO-P treats the problem as a finite-arm bandit problem with respect to the cubes in the dyadic partition at $t$; each cube possesses a confidence bound. The algorithm then chooses a best cube according to UCB, and chooses an arm uniformly at random within the chosen cube. 
Before formulating our strategy, we put forward several functions that summarize cube statistics. 






Let $\mathcal{Q}_{t}$ be the collection of dyadic cubes of $[0,1)^d$ at time $t$ ($t \ge 1$). Let $(a_1, y_1, a_2, y_2 , \cdots, a_{t }, y_{t } )$ be the observations received up to time $t$. We define
\begin{itemize}[leftmargin=15pt]
    \vspace{-0.3cm} 
    \item the \textit{cube count} $n_{t}  : \mathcal{Q}_t  \rightarrow    \mathbb{R}$, such that for $q \in \mathcal{Q}_t$
   	\vspace{-7pt}
    \begin{align}
    	n_{t} (q) &:= \sum_{i=1}^{t - 1 } \mathbb{I}_{[a_i \in q ]};  \;\;
        \widetilde{n}_{t} (q) := \max( 1, n_{t} (q) ). \label{eq:count}
        \vspace{-0.5cm}
    \end{align}
    \vspace{-0.5cm} 
    \item the \textit{cube average} $m_{t}: \mathcal{Q}_t \rightarrow \mathbb{R}$, such that for $q \in \mathcal{Q}_t$
  	\vspace{-5pt}
    \begin{align}
    m_{t}(q) := 
    \begin{cases}
    \frac{\sum_{i = 1 }^{ t - 1 } y_i \mathbb{I}_{ [a_i \in q ]} }{ n_{t} (q) } , 
    \text{ if }  n_{t}(q) > 0;  \label{eq:average}
    \\
    0, \quad \text{otherwise}. 
    \end{cases} 
    \end{align} 
    \vspace{-20pt}
\end{itemize}




At time $t$, based on the partition $\mathcal{Q}_{t-1}$ and observations $ (a_1, y_1, a_2, y_2, \cdots, a_{t-1}, y_{t-1} ) $, our bandit algorithm picks a cube (and plays an arm within the cube uniformly at random). More specifically, the algorithm picks
\begin{align}
Q_t &\in \arg \max_{ q \in \mathcal{Q}_t } U_t(q), \quad \label{eq:ucb-u} \text{where} \\
U_t(q) :=& m_{t} (q) + H_t (q) + J(q) \label{eq:ucb-index}, \\ 
H_t (q) :=& \BoundHoeffding{t}{q} , \nonumber \\
J(q) :=& \left\lfloor \BoundMeasure{  \mu (q) / \eta } \right\rfloor_+ ,\nonumber \\
\Psi :=& \PsiBound, \label{eq:def-Psi} 
\end{align} 
where $T$ is the time horizon, $D_\mathpzc{E}$ is the a.s$.$ bound on the noise, $\epsilon$ and $\eta$ are algorithm parameters (to be discussed in more detail later),
and $\left\lfloor z \right\rfloor_+ = \max\{ 0 , z \}.$
Here $ \Psi $ is the ``Effective Bound'' of the expected reward, and $\eta$ controls minimal cube size in the partition $\mathcal{Q}_t$ (Proposition \ref{prop:cube-number} in Appendix \ref{app:pf-cube-number}). 
All these quantities will be discussed in more detail as we develop our algorithm. 



After playing an arm and observing reward,
we update the partition into a finer one if needed.  
Next, we discuss our partition refinement rules and the tie-breaking mechanism. 

\textbf{Partition Refinement: }
We start with $\mathcal{Q}_0 = \{ [0,1)^d \}$. 
At time $t$, we split cubes in $ \mathcal{Q}_{t-1}$ to construct $\mathcal{Q}_{t}$ so that 
the following is satisfied for any $q \in \mathcal{Q}_t$
\begin{align}
    &H_t ( q) \ge J ( q),\quad \text{or equivalently} \nonumber \\
    &\BoundHoeffding{t}{q} \ge \left\lfloor \BoundMeasure{ \mu(q) / \eta } \right\rfloor_+ .   \label{eq:partition-rule} 
\end{align}
In (\ref{eq:partition-rule}), the left-hand-side does not decrease as we make splits (the numerator remains constant while the denominator can only decrease), while the right-hand-side decreases until it hits zero as we make more splits. Thus  (\ref{eq:partition-rule}) can always be satisfied with additional splits.

\textbf{Tie-breaking: } We break down our tie-breaking mechanism into two steps. In the first step, we choose a cube $Q_t \in \mathcal{Q}_{t-1}$ such that: 
\begin{align} 
    Q_t \in \arg \max_{q \in \mathcal{Q}_{t-1}}  \label{eq:select-best-cube}
    U_t (q) . 
\end{align}
After deciding from which cube to choose an arm, we uniformly randomly play an arm $A_t$ within the cube $Q_t$. If measure $ \mu $ is non-uniform, we play arm $A_t$, so that for any subset $S \subset Q_t$, 
    $\mathbb{P} (A_t \in S) = \frac{\mu (S)}{ \mu (Q_t) }. $

The random variables $ \{ (Q_{t^\prime}, A_{t^\prime}, Y_{t^\prime}) \}_{t^\prime} $ (\textit{cube selection}, \textit{arm selection}, \textit{reward}) describe all randomness in the learning process up to time $t$. We summarize this strategy in Algorithm \ref{alg}. Analysis of Algorithm \ref{alg} is found in Section \ref{sec:theory}, which also provides some tools for handling $\delta$-regret. Then in Section \ref{sec:zooming}, we provide an improved algorithm that exhibits a stronger performance guarantee. 

\begin{algorithm}[h!]
    \caption{Bandit-BMO-Partition (Bandit-BMO-P)}  
    \label{alg}
    \begin{algorithmic}[1] 
        \STATE Problem intrinsics: $\mu (\cdot)$, $ 
        D_{\mathpzc{E}} $, $d$, $M_d$. 
        \STATE \COMMENT{$\mu (\cdot)$ is the Lebesgue measure. $D_{\mathpzc{E}}$  bounds the noise.} 
        \STATE \COMMENT{$d$ is the dimension of the arm space.} 
        \STATE \COMMENT{$M_d$ is the doubling constant of the arm space.} 
        \STATE Algorithm parameters: $\eta > 0$, $\epsilon > 0$, $T$. 
        \STATE \COMMENT{$T$ is the time horizon. $\epsilon$ and $\eta$ are parameters.}
        \FOR{$t = 1, 2, \dots, T$}
            \STATE Let $m_{t}$ and $n_{t}$ be defined as in (\ref{eq:average}) and (\ref{eq:count}). 
            \STATE Select a cube $Q_t \in \mathcal{Q}_{t}$ such that:
            \begin{align*} 
                Q_t \in \arg \max_{q \in \mathcal{Q}_{t-1}} U_t (l),
            \end{align*}
            where $U_t$ is defined in (\ref{eq:ucb-index}).
            \STATE Play arm $A_t \in Q_t$ uniformly at random. Observe $Y_t$. 
            \STATE Update the partition $\mathcal{Q}_{t}$ to $\mathcal{Q}_{t+1}$ according to (\ref{eq:partition-rule}).
        \ENDFOR 
    \end{algorithmic}
\end{algorithm}

\subsection{Regret Analysis of Bandit-BMO-P} \label{sec:theory}
In this section we provide a theoretical guarantee on the algorithm. 
We will use capital letters (e.g., $Q_t, A_t, Y_t$) to denote random variables, and use lower-case letters (e.g. $a, q$) to denote non-random quantities, unless otherwise stated. 


\begin{theorem} \label{thm:regret-bound}
    Fix any $T$. With probability at least $ 1 - {2 \epsilon}  $, for any $\delta > |\mathcal{Q}_T| \eta$ such that $\delta$ is $f$-admissible, the total $\delta$-regret for Algorithm \ref{alg} up to time $T$ satisfies 
    \begin{align}
        \sum_{t=1}^T r_{t}^\delta  
        \lesssim_d  \widetilde{\mathcal{O}} \left( \sqrt{ T | \mathcal{Q}_T | } \right)  
         , 
    \end{align}
    where the $ \lesssim_d $ sign omits constants that depends on $d$, and $|\mathcal{Q}_T|$ is the cardinality of $\mathcal{Q}_T$. 
\end{theorem}

From uniform tie-breaking, we have 
\begin{align}
    &\mathbb{E} [f(A_t) | \mathcal{F}_t] = \frac{1}{ \mu ( Q_t ) } \int_{ a \in  Q_t} f (a ) { \, d  a } = \left< f \right>_{ Q_t },  \label{eq:arm-exp-to-cube-mean} \\
    &\mathcal{F}_t = \sigma (Q_1, A_1, Y_1, \cdots, Q_{t-1}, A_{t-1}, Y_{t-1}, Q_t),  \label{eq:def-filtration}
\end{align} 
where $ \mathcal{F}_t $ is the $\sigma$-algebra generated by random variables $Q_1, A_1, Y_1, \cdots, Q_{t-1}, A_{t-1}, Y_{t-1}, Q_t$ -- all randomness right after selecting cube $Q_t$. 
At time $t$, the expected reward is the mean function value of the selected cube. 

The proof of the theorem is divided into two parts. In \textbf{Part I}, we show that some ``good event'' holds with high probability. In \textbf{Part II}, we bound the $\delta$-regret under the ``good event.'' 

\textbf{Part I: } 
For $t \le T$, and $q \in \mathcal{Q}_t$, we define 
\begin{align}
	\mathcal{E}_t (q) &:= 
	\begin{Bmatrix*}[l] 
	\left| \left< f \right >_{ q } - m_{t} ( q ) \right| 
	\le H_t ( q ) 
	\end{Bmatrix*} \label{eq:good-event}, \\
	 H_t (q) &= \BoundHoeffding{t}{q}. \label{eq:Ht-for-prob}
\end{align} 
In the above, ${\mathcal{E}}_t (q) $ is essentially saying that the empirical mean within a cube $q$ concentrates to $\left< f \right>_{q}$. 
Lemma \ref{lem:bound-on-event} shows that ${\mathcal{E}}_t (q )$ happens with high probability for any $t$ and $q$. 

\begin{lemma} \label{lem:bound-on-event}
	With probability at least $ 1 -  \frac{\epsilon}{T}  $, the event $ \mathcal{E}_t (q) $ holds for any $q \in \mathcal{Q}_t$ at any time $t$. 
\end{lemma} 

To prove Lemma \ref{lem:bound-on-event}, we apply a variation of Azuma's inequality \citep{vu2002concentration, tao2015random}. We also need some additional effort to handle the case when a cube $q$ contains no observations. 
The details are in Appendix \ref{app:pf-bound-on-event}. 

\textbf{Part II: } 
Next, we link the $\delta$-regret to the $J (q)$ term. 


\begin{lemma} 
    \label{lem:exist-good-cube-arm}
    Recall $J(q) = \BoundMeasurecube{q}$. For any partition $\mathcal{Q}$ of $[0,1)^d$, there exists $q \in \mathcal{Q}$, such that 
    \begin{align} 
        f^\delta -
        \AverageValue{q} \le 
         J(q), \label{eq:con-measure-max}
    \end{align}
    for any $f$-admissible $\delta > \eta |\mathcal{Q}|$, where $|\mathcal{Q}|$ is the cardinality of $\mathcal{Q}$.  
\end{lemma} 


In the proof of Lemma \ref{lem:exist-good-cube-arm}, we suppose, in order to get a contradiction, that there is no such cube. Under this assumption, there will be contradiction to the definition of $f^\delta$. 


By Lemma \ref{lem:exist-good-cube-arm}, there exists a ``good'' cube $\widetilde{q}_{t}$ (at any time $t \le T$), such that (\ref{eq:con-measure-max}) is true for $\widetilde{q}_{t}$. 
Let $\delta$ be an arbitrary number satisfying (1) $\delta > | \mathcal{Q}_T | \eta$ and (2) $\delta$ is $f$-admissible. 
Then under event $ {\mathcal{E}} (\widetilde{q}_{t})$, 
\begin{align}
    f^\delta  
    =& \left( f^\delta - \left<f \right>_{\widetilde{q}_{t}} \right) + \left( \left<f \right>_{\widetilde{q}_{t}}  - m_{t} (\widetilde{q}_{t}) \right) + m_{t} (\widetilde{q}_t) \nonumber \\
    \overset{\text{\pgftextcircled{1}}}{\le}& 
    J(\widetilde{q}_t)  +  H_t (\widetilde{q}_t) +  m_t (\widetilde{q}_t), \label{eq:bound-f-delta-JH}
\end{align} 
where \text{\pgftextcircled{1}} uses Lemma \ref{lem:exist-good-cube-arm} for the first brackets and Lemma \ref{lem:bound-on-event} (with event ${\mathcal{E}}_{t} (\widetilde{q}_{t})$) for the second brackets. 




The event where all ``good'' cubes and all cubes we select (for $t \le T$) have nice estimates, namely $ \left( \bigcap_{t = 1}^T {\mathcal{E}}_{t} (\widetilde{q}_{t}) \right) \bigcap \left( \bigcap_{t = 1}^T {\mathcal{E}}_{t} (Q_{t} ) \right),$  occurs with probability at least $1 - 2 \epsilon$. This result comes from Lemma \ref{lem:bound-on-event} and a union bound, and we note that $\mathcal{E}_{t} (q)$ depends on $\epsilon$ (and $T$), as in (\ref{eq:Ht-for-prob}).  
Under this event, from (\ref{eq:good-event}) we have
    $ \left| \left< f \right>_{Q_t} - m_t (Q_t) \right| \le H_t (Q_t) . $
This and (\ref{eq:arm-exp-to-cube-mean}) give us 
\begin{align}
    \mathbb{E} \left[ f(A_t) | \mathcal{F}_t \right] \h = \h \left< f \right>_{Q_t} \h \ge \h m_t (Q_t) \h - \h H_t (Q_t) \h .
    \label{eq:for-E-Ft} 
\end{align} 
We can then use the above to get, under the ``good event'',
\begin{align}
    &f^\delta - \mathbb{E}[ f(A_{t }) | \mathcal{F}_t ] \nonumber\\
    \overset{\pgftextcircled{1}}{\le}& m_{t } (\widetilde{q}_{t}) + H_t(  \widetilde{q}_{t} ) + J (  \widetilde{q}_{t} ) 
    - m_{t } (Q_{t}) + H_t(  Q_{t} ) 
    \nonumber \\ 
    \overset{\pgftextcircled{2}}{\le}&  m_{t } (Q_{t} ) \h + \h H_t(  Q_{t} ) \h 
    + \h J (  Q_{t} ) 
    - m_{t } (Q_{t}) \h + \h H_t(  Q_{t} ) \h 
    \nonumber \\
    =& 2 H_t(  Q_{t} ) + J(  Q_{t} ) 
    \le 3 H_t(  Q_{t} ), \label{eq:get-single-regret}
\end{align}
where \pgftextcircled{1} uses (\ref{eq:bound-f-delta-JH}) for the first three terms and (\ref{eq:for-E-Ft}) for the last three terms, \pgftextcircled{2} uses that $ U_t(Q_t) \ge U_t(\widetilde{q}_t) $ since $Q_t$ maximizes the index $U_t (\cdot)$ according to (\ref{eq:select-best-cube}), and the last inequality uses the rule (\ref{eq:partition-rule}). 

Next, we use Lemma \ref{lem:point-scattering} 
to link the number of cubes up to a time $t$ to the Hoeffding-type tail bound in (\ref{eq:get-single-regret}). Intuitively, this bound (Lemma \ref{lem:point-scattering}) states that the numbers of points within the cubes grows fast enough to be bounded by a function of the number of cubes. 

\begin{lemma} 
\label{lem:point-scattering}
We say a partition $\mathcal{Q}$ is finer than a partition $\mathcal{Q}'$ if for any $q \in \mathcal{Q}$, there exists $q' \in \mathcal{Q}'$ such that $q \subset q' $.
Consider an arbitrary sequence of points $x_1, x_2, \cdots, x_t,  \cdots$ in a space $\mathcal{X}$, and a sequence of partitions $\mathcal{Q}_1, \mathcal{Q}_2, \cdots $ of $\mathcal{X}$ such that $\mathcal{Q}_{t+1}$ is finer than $\mathcal{Q}_{t}$ for all $t = 1,2,\cdots, T-1$. Then for any $T$, and $\{ q_t \in \mathcal{Q}_t \}_{t=1}^{T}$,
\vspace{-6pt} 
\begin{align} 
    \sum_{t =  1 }^{T}  \frac{1}{ \widetilde{n}_{t } ( q_{t } ) } &\le e | \mathcal{Q}_T |  \log \left( 1 + (e - 1) \frac{T }{  | \mathcal{Q}_T |  } \right) , \label{eq:point-scattering-gp} 
\end{align}
where $\widetilde{n}_{t}$ is defined in (\ref{eq:count}) (using points $x_1, x_2, \cdots$), and $|\mathcal{Q}_t|$ is the cardinality of partition $\mathcal{Q}_t$. 
\end{lemma}

A proof of Lemma \ref{lem:point-scattering} is in Appendix \ref{app:point-scattering-gp}. We can apply Lemma \ref{lem:point-scattering} and the Cauchy-Schwarz inequality to  (\ref{eq:get-single-regret}) to prove Theorem \ref{thm:regret-bound}. The details can be found in Appendix \ref{app:pf-thm-p}. 

\section{Achieve Poly-log Regret via Zooming} \label{sec:zooming}
In this section we study an improved version of the previous section that uses the Zooming machinery \citep{kleinberg2008multi,slivkins2014contextual} and inspirations from \citet{bubeck2011x}. Similar to Algorithm \ref{alg}, this algorithm runs by maintaining a set of dyadic cubes $\Q_t$. 

In this setting, we divide the time horizon into episodes. In each episode $t$, we are allowed to play multiple arms, and all arms played can incur regret. 
This is also a UCB strategy, and the index of $q \in \Q_t$ is defined the same way as (\ref{eq:ucb-index}):
\begin{equation}
	U_t(q):= m_t(q) + H_t(q) + J(q) \label{eq:parent-index} 
\end{equation}
Before we discuss in more detail how to select cubes and arms based on the above index $U_t (\cdot)$, we first describe how we maintain the collection of cubes. 
Let $\mathcal{Q}_t$ be the collection of dyadic cubes at episode $t$. We first define \textit{\textbf{terminal cubes}}, which are cubes that do not have sub-cubes in $\Q_t$. More formally,
a cube $Q \in \mathcal{Q}_t$ is a terminal cube if there is no other cube $Q' \in \mathcal{Q}_t$ such that $Q' \subset Q$. A \textit{\textbf{pre-parent cube}} is a cube in $\Q_t$ that ``directly'' contains a terminal cube: For a cube $Q \in \mathcal{Q}_t$, if $Q$ is a direct super cube of any terminal cube, we say $Q$ is a pre-parent cube. Finally, for a cube $Q \in \Q_t$, if $Q$ is a pre-parent cube and no super cube of $Q$ is a pre-parent cube, we call $Q$ a \textit{\textbf{parent cube}}. Intuitively, no ``sibling'' cube of a parent cube is a terminal cube. As a consequence of this definition, a parent cube cannot contain another parent cube. 
Note that some cubes are none of the these three types of cubes. 
Figure \ref{fig:cube-example} gives examples of terminal cubes, pre-parent cubes and parent cubes. 

\textbf{Algorithm Description }

Pick \textbf{\textit{zooming rate}} $\alpha \in \left(0 , \frac{ (\Psi + D_{\mathpzc{E}}) \sqrt{2 \log (2T^2 / \epsilon)} }{\log ( M_d / \eta)} \right]$. The collection of cubes grows following the rules below: 
	\textbf{(1)}
	Initialize $\mathcal{Q}_0 = \{ [0,1)^d \}$ and $[0,1)^d$. Warm-up: play $n_{warm}$ arms uniformly at random from $[0,1)^d$ so that 
	\begin{align} 
	\begin{cases} 
	    \frac{ \left( {\Psi}  + {D_{\mathpzc{E}} } \right) \sqrt{ \HoeffTSq }  }{ \sqrt{ n_{warm}  } } \ge \alpha \BoundMeasure{ \frac{ M_d }{\eta} } \\
	    \frac{ \left( {\Psi}  + {D_{\mathpzc{E}} } \right) \sqrt{ \HoeffTSq }  }{ \sqrt{ n_{warm} +1 } } < \alpha \BoundMeasure{ \frac{ M_d }{\eta} }
	\end{cases}
	 \label{eq:warm-up}. 
	\end{align}  
	
	\textbf{(2)}
	After episode $t$ ($t = 1,2,\cdots, T$), 
	ensure
	\begin{align} 
	\BoundHoeffding{t}{Q^{ter}} \ge \alpha \BoundMeasure{ \frac{ M_d \mu ( Q^{ter} ) }{\eta} } \label{eq:split-rule-zooming} 
	\end{align}  
	for any terminal cube $Q^{ter}$. 
	If (\ref{eq:split-rule-zooming}) is violated for a terminal cube $Q^{ter}$, we include the $M_d$ direct sub-cubes of $Q^{ter}$ into $\Q_t$. Then $Q^{ter}$ will no longer be a terminal cube and the direct sub-cubes of $Q^{ter}$ will be terminal cubes. 
	We repeatedly include direct sub-cubes of (what were) terminal cubes into $\mathcal{Q}_t$, until all terminal cubes satisfy  (\ref{eq:split-rule-zooming}). We choose $\alpha$ to be smaller than $\frac{ (\Psi + D_{\mathpzc{E}}) \sqrt{2 \log (2T^2 / \epsilon)} }{\log ( M_d / \eta)}$ so that (\ref{eq:split-rule-zooming}) can be satisfied with $\widetilde{n}_t (Q^{ter}) = 1$ and $ \mu (Q^{ter}) = 1 $. 
	

As a consequence, any non-terminal cube $Q^{par}$ (regardless of whether it is a pre-parent or parent cube) satisfies: 
\begin{align}
	\BoundHoeffding{t}{Q^{par}} < \alpha \BoundMeasure{ \frac{ M_d \mu (  Q^{par} ) }{\eta} } .  \label{eq:split-rule-zooming-parent}
\end{align} 


\begin{figure}[ht!]
	\centering 
	\includegraphics[scale = 0.35]{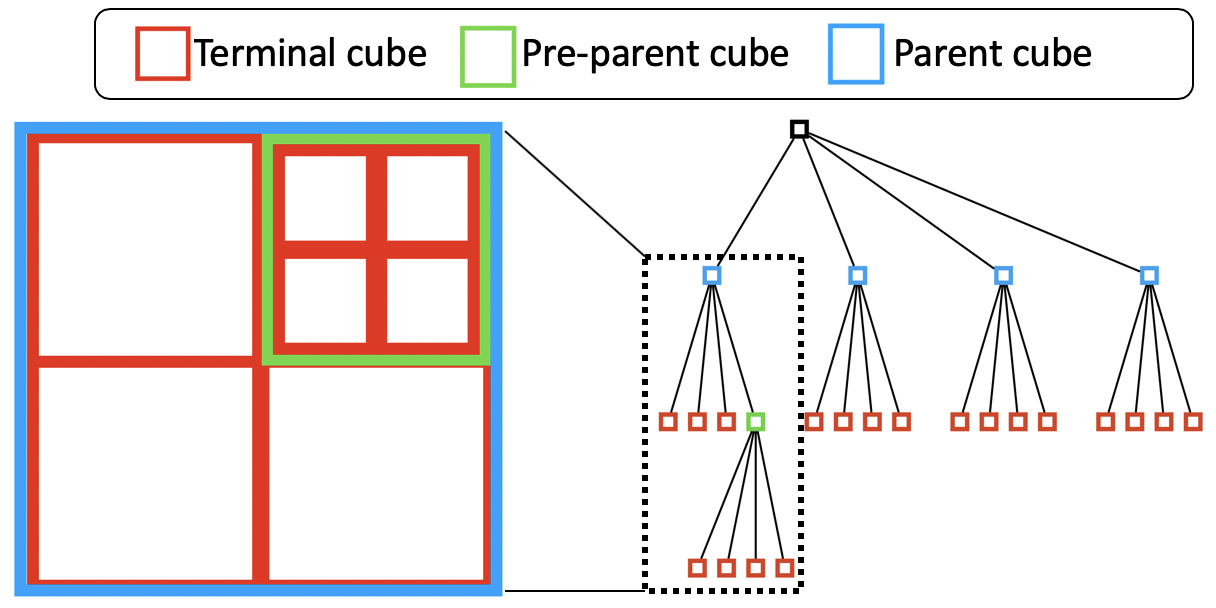} 
	\caption{Example of terminal cubes, pre-parent and parent cubes. 
		\label{fig:cube-example}} 
\end{figure} 

After the splitting rule is achieved, we select a parent cube. Specifically $Q_t$ is chosen to maximize the following index: 
\begin{align}
	Q_t \in \arg\max_{q \in \mathcal{Q}_t, q \text{ is a parent cube}} U_t (q). \nonumber
\end{align}

Within each \textit{direct sub-cube} of $Q_t$ (either pre-parent or terminal cubes), we uniformly randomly play one arm. In each episode $t$, $M_d$ arms are played. This algorithm is summarized in Algorithm \ref{alg:z}. 

\begin{algorithm}[h!]
    \caption{Bandit-BMO-Zooming (Bandit-BMO-Z)}  
    \label{alg:z}
    \begin{algorithmic}[1] 
        \STATE Problem intrinsics: $\mu (\cdot)$,  $ 
        D_{\mathpzc{E}} $,  $d$, $M_d$. 
        \STATE \COMMENT{$\mu (\cdot)$,  $ 
        D_{\mathpzc{E}} $,  $d$, $M_d$ are same as those in Algorithm \ref{alg}.}
        \STATE Algorithm parameters: $\eta, \epsilon, T > 0$, and $\alpha \in \left(0,  \frac{ (\Psi + D_{\mathpzc{E}}) \sqrt{2 \log (2T^2 / \epsilon)} }{\log ( M_d / \eta)} \right]$. 
        \STATE \COMMENT{$\eta $, $\epsilon$, $T$ are same as those in Algorithm \ref{alg}. $\alpha$ is the zooming rate.}
        \STATE Initialize: let $\Q_0 \hspace{-2pt} = \hspace{-2pt} [0,1)^d$. Play warm-up phase (\ref{eq:warm-up}). 
        \FOR{episode $t = 1, 2, \dots, T$}
            \STATE Let $m_{t}$, $n_{t}$, $U_t $ be defined as in (\ref{eq:average}), (\ref{eq:count}) and (\ref{eq:parent-index}). 
            \STATE Select parent cube $Q_t \in \mathcal{Q}_{t}$ such that: 
            \begin{align*} 
                Q_t \in \arg \max_{q \in \mathcal{Q}_{t}, \; q \text{ is a parent cube. }} U_t (q). 
            \end{align*}
			\FOR{$j = 1, 2, \dots, M_d$}
				\STATE Locate the $j$-th direct sub-cube of $Q_t$: $Q_j^{sub}$. 	
				\STATE Play $A_{t,j } \in Q_j^{sub}$ uniformly at random, and observe $Y_{t,j}$. 
			\ENDFOR 
            \STATE Update the collection of dyadic cubes $\mathcal{Q}_{t}$ to $\mathcal{Q}_{t+1}$ according to  (\ref{eq:split-rule-zooming}). 
        \ENDFOR 
    \end{algorithmic}
\end{algorithm}

\textbf{Regret Analysis: }
For the rest of the paper, we define $$\mathcal{F}_t := \sigma \Bigg( \bigg\{ Q_{t'}, \{ A_{t',j} \}_{j=1}^{M_d}, \{ Y_{t',j} \}_{j=1}^{M_d} \bigg\}_{t'=1}^{t-1} , Q_t \Bigg) , $$ which is the $\sigma$-algebra describing all randomness right after selecting the parent cube for episode $t$. We use $\mathbb{E}_t$ to denote the expectation conditioning on $\mathcal{F}_t$.  
We will show Algorithm \ref{alg:z} achieves $\widetilde{\mathcal{O}} \left( poly\text{-}log ( T ) \right)$ $\delta$-regret with high probability (formally stated in Theorem \ref{thm:z}).

Let $A_{t,i}$ be the $i$-th arm played in episode $t$. Let us denote $\Delta_{t,i}^\delta := f^\delta - \mathbb{E}_t [ f(A_{t,i}) ]$. Since each $A_{t,i}$ is selected uniformly randomly within a direct sub-cube of $Q_t$, we have 
\begin{align}
	\sum_{i=1}^{M_d} \mathbb{E}_t [f(A_{t,i})] = M_d \left< f \right>_{Q_t},
\end{align}
where $\mathbb{E}_t$ is the expectation conditioning on all randomness before episode $t$. Using the above equation, for any $t$, 
\begin{align}
	\sum_{i=1}^{M_d} \Delta_{t,i }^\delta = M_d ( f^\delta - \left< f \right>_{Q_t} ). \label{eq:Delta-gap}
\end{align} 
\vspace{-12pt}

The quantity $\sum_{i=1}^{M_d}\Delta_{t,i}^\delta$ is the $\delta$-regret incurred during episode $t$. We will bound (\ref{eq:Delta-gap}) using tools in Section \ref{sec:algo}. In order to apply Lemma \ref{lem:exist-good-cube-arm}, we need to show that the parent cubes form of partition of the arm space (Proposition \ref{prop:parent-partition}). 



\begin{proposition} \label{prop:parent-partition}
	At any episode $t$, the collection of parent cubes forms a partition of the arm space. 
\end{proposition}

Since the parent cubes in $\Q_t $ form a partition of the arm space, we can apply Lemma \ref{lem:exist-good-cube-arm} to get the following. 
For any episode $t$, there exists a parent cube $q_t^{\max}$, such that
\begin{align} 
	f^\delta \le& \left< f \right>_{q_t^{\max}} + \BoundMeasure{  \mu( q_t^{\max} ) / \eta }. \label{eq:f-best-parent-bound}
\end{align} 
Let us define $\widetilde{\mathcal{E}}_T := \left( \bigcap_{t = 1}^T {\mathcal{E}}_{t} ({q}_{t}^{\max} ) \right) \bigcap \left( \bigcap_{t = 1}^T {\mathcal{E}}_{t} (Q_{t} ) \right)$, where $ \mathcal{E}_{t} (q_t^{\max})$ and $ \mathcal{E}_{t} (Q_t)$ are defined in (\ref{eq:good-event}). By Lemma \ref{lem:bound-on-event} and another union bound, we know the event $\widetilde{\mathcal{E}}_T$ happens with probability at least $1 - {2 \epsilon}$. 

 
%
%




Since each episode creates at most a constant number of new cubes, we have $|\mathcal{Q}_t| = \mathcal{O} (t)$. Using the argument we used for (\ref{eq:get-single-regret}), we have that at any $t \le T$, for any $\delta > \eta |\mathcal{Q}_t| $ that is $f$-admissible, under event $\widetilde{\mathcal{E}}_T$, 
\begin{align}
	& \sum_{i=1}^{M_d} \Delta_{t,i}^\delta 
	= {M_d} \left( f^\delta - \left<f\right>_{Q_t} \right) \label{eq:use-Delta-gap} \\
	\le & M_d \left( 2 \BoundHoeffding{t}{Q_t} + \log \left( \frac{ M_d \mu(Q_t) }{\eta} \right)  \right) \nonumber \\
	\le & M_d (1 + 2 \alpha) \log \left( \frac{ M_d \mu (Q_t) }{ \eta } \right),  \label{eq:bound-Delta} 
\end{align} 
where (\ref{eq:use-Delta-gap}) uses (\ref{eq:Delta-gap}) and the last inequality uses (\ref{eq:split-rule-zooming-parent}). 

Next, we extend some definitions from \citet{kleinberg2008multi}, to handle the $\delta$-regret setting. 
Firstly, we define the set of $(\lambda,\delta)$-optimal arms as
\begin{align}
	\mathcal{X}_\delta ( \lambda ) := \left(  \bigcup  \{ Q \subset [ 0,1 ]^d  : f^\delta - \left< f \right>_Q \le \lambda \} \right) .
	\vspace{-0.1cm}
\end{align}

We also need to extend the definition of zooming number \citep{kleinberg2008multi} to our setting. We denote by $N_\delta (\lambda, \xi)$ the number of cubes of edge-length $\xi$ needed to cover the set $ \mathcal{X}_\delta (\lambda) $. Then we define the $ ( \delta, \eta )$-Zooming Number with zooming rate $\alpha$ as 
\begin{align}
	\hspace{-0.25cm} \widetilde{ N }_{\delta, \eta, \alpha} 
	\hspace{-2pt}:= 
	\hspace{-2pt} 
	\sup_{ \lambda \in \left(  \eta^{\frac{1}{d}}, 1 \right] }  \hspace{-2pt} 
	N_\delta \hspace{-2pt}\left( (1 + 2 \alpha) \log  \left( M_d \lambda^d / \eta  \right) , \lambda \right),  \label{eq:def-N-tilde} 
\end{align} 
where 
$ N_\delta \hspace{-2pt}\left( (1 + 2 \alpha ) \log  \left( M_d \lambda^d / \eta  \right) , \lambda \right)  $ is the number of cubes of edge-length $\lambda$ needed to cover $\mathcal{X}_\delta ( (1 + 2 \alpha) \log ( M_d \lambda^d / \eta ) )$. 
The number $ \widetilde{ N }_{\delta, \eta, \alpha} $ is well-defined. This is because the $ \mathcal{X}_\delta ( (1 + 2 \alpha ) \log ( M_d \lambda^d / \eta ) ) $ is a subspace of $(0,1]^d$, and number of cubes of edge-length $> \eta^{\frac{1}{d}}$ needed to cover $(0,1]^d$ is finite. 
Intuitively, the idea of zooming is to use smaller cubes to cover more optimal arms, and vice versa. BMO properties convert between units of reward function and units in arm space. 

We will regroup the $\Delta_{t,i}$ terms to bound the regret. To do this, we need the following facts, whose proofs are in Appendix \ref{app:pf-prop-regroup}. 

\begin{proposition} \label{prop:regroup}
	Following the Zooming Rule (\ref{eq:split-rule-zooming}), we have  
\textbf{1.} Each parent cube of measure $\mu$ is played at most $ \maxHitNumber{ \mu / \eta  } $ episodes. \\
\textbf{2.} Under event $\widetilde{\mathcal{E}}_T$, each parent cube $Q_t$ selected at episode $t$ is a subset of $\mathcal{X}_\delta \left( (1 + 2 \alpha)  \log \left( M_d \mu (Q_t) / \eta \right)  \right)$. 
\end{proposition} 

%
For cleaner writing, we set $\eta = 2^{-dI}$ for some positive integer $I$, and assume the event $\widetilde{\mathcal{E}}_T$ holds. By Proposition \ref{prop:regroup}, we can regroup the regret in a similar way to that of \citet{kleinberg2008multi}. Let $\mathcal{K}_{i}$ be the collection of selected parent cubes such that for any $ Q \in  \mathcal{K}_{i} $, $ \mu (Q) = 2^{-di} $ (dyadic cubes are always of these sizes). 
The sets $\mathcal{K}_{i}$ regroup the selected parent cubes by their size. 
By Proposition \ref{prop:regroup} (item 2), we know each parent cube in $\mathcal{K}_i$ is a subset of $\mathcal{X}_\delta \left( (1 + 2 \alpha) \log \left( M_d 2^{-di} / \eta \right)  \right) $. 
Since cubes in $\mathcal{K}_{i}$ are subsets of $\mathcal{X}_\delta \left( (1 + 2 \alpha) \log \left( M_d 2^{-di} / \eta \right) \right)$ and cubes in $\mathcal{K}_i$ are of measure $2^{-di}$, we have 
\begin{align}
	|\mathcal{K}_{i} | \le N_\delta \left( (1 + 2 \alpha)  \BoundMeasure{ M_d 2^{-di} /\eta } , 2^{-i} \right) , \label{eq:bound-Kv}
\end{align}
where $|\mathcal{K}_{i} |$ is the number of cubes in $\mathcal{K}_{i} $. 
For a cube $Q$, let $S_Q$ be the episodes where $Q$ is played. 
With probability at least $1 - 2 \epsilon $, we can regroup the regret as
\begin{align} 
	&\sum_{t=1}^T \sum_{i=1}^{M_d} \Delta_{t,i}^\delta 
	\le \sum_{t=1}^T (1 + 2 \alpha) M_d  \log \left( M_d \mu(Q_t) / \eta \right) \label{eq:use-Delta-bound} \\
	&\le \crazysum \sum_{Q \in \mathcal{K}_{i} } \sum_{t \in S_Q} (1 + 2\alpha ) M_d  \BoundMeasure{ M_d 2^{-di} / \eta }  \label{eq:regroup-sum} , 
\end{align} 
where (\ref{eq:use-Delta-bound}) uses (\ref{eq:bound-Delta}), (\ref{eq:regroup-sum}) regroups the sum as argued above. 
Using Proposition \ref{prop:regroup}, we can bound (\ref{eq:regroup-sum}) by: 
\vspace{-5pt}
\begin{align} 
	& \crazysum \sum_{Q \in \mathcal{K}_{i} } \sum_{t \in S_Q} (1 + 2 \alpha ) M_d  \BoundMeasure{ M_d 2^{-di} / \eta } \nonumber \\
	\le& \crazysum \sum_{Q \in \mathcal{K}_{i} } | S_Q | (1 + 2\alpha) M_d  \BoundMeasure{ \dfrac{ M_d 2^{-di} }{ \eta } }  \nonumber \\
	\overset{\textcircled{\raisebox{-0.9pt}{1}}}{\le}& \crazysum \hspace{-1pt} \sum_{Q \in \mathcal{K}_{i} } \hspace{-4pt} \maxHitNumber{  2^{-di} / \eta  } \\
	&\cdot ( 1 \hspace{-1.5pt} +  \hspace{-1.5pt} 2\alpha) M_d  \log \hspace{-2pt} \left( \hspace{-1.5pt} \dfrac{ M_d 2^{-di} }{ \eta }  \hspace{-1.5pt} \right) \nonumber  \\
	\le& \crazysum N_\delta \left( (1 + 2\alpha) \log  \left( M_d 2^{-di} / \eta  \right) , 2^{-d i} \right) \nonumber \\
	& \cdot \maxHitNumber{ 2^{-di} / \eta  } \hspace{-1.5pt} \cdot \hspace{-1.5pt} (1 + 2\alpha) M_d  \BoundMeasure{ \hspace{-1.5pt} \frac{ M_d 2^{-di} }{ \eta} \hspace{-1.5pt} } \label{eq:use-bound-Kv} \\ 
	\le& \frac{2 (1 + 2 \alpha) M_d ( \Psi + D_{\mathpzc{E}})^2  }{ \alpha^2 } \widetilde{N}_{\delta, \eta, \alpha} \nonumber \\
	& \cdot \log (2T^2 / \epsilon )  \crazysum   \frac{ \log ( M_d 2^{-di} / \eta ) }{  \left[ \log ( 2^{-di} / \eta ) \right]^2 } \nonumber , 
\end{align} 


where \textcircled{\raisebox{-0.9pt}{1}} uses item 1 in Proposition \ref{prop:regroup}, (\ref{eq:use-bound-Kv}) uses (\ref{eq:bound-Kv}). 
Recall $\eta = 2^{-dI}$ for some positive integer $I$. We can use the above to prove Theorem \ref{thm:z}, by using $\eta = 2^{-dI}$ and
\begin{align}
   &\hspace*{-8pt}\crazysum   \frac{ \log ( M_d 2^{-di} / \eta ) }{  \left[ \log ( 2^{-di} / \eta ) \right]^2 } 
    \hspace*{-1pt}=\hspace*{-1pt} \crazysum   \frac{ \log M_d  }{  \left[ \log \dfrac{ 2^{-di} }{ \eta } \right]^2 }\hspace*{-1pt} +\hspace*{-1pt} \crazysum   \frac{ 1 }{  \log \dfrac{ 2^{-di} }{ \eta }  }  \nonumber \\
    =& \crazysum   \frac{ \log M_d  }{  d^2 \left( \log 2 \right)^2 (I - i)^2 } + \crazysum   \frac{ 1 }{ d ( \log 2 ) (I - i)  }  \label{eq:last-one} \\
    =& \mathcal{O} \left( 1 \right) +  \mathcal{O} \left( \log I \right), \nonumber \\
    =&  \mathcal{O} \left( \log  \log (1/\eta) \right), \nonumber
\end{align}
where the first term in (\ref{eq:last-one}) is $\mathcal{O} (1)$ since $\sum_{i=1}^\infty \frac{1}{i^2} = \mathcal{O} (1)$ and the second term in (\ref{eq:last-one}) is $\mathcal{O} \left( \log I \right)$ by the order of a harmonic sum. The above analysis gives Theorem \ref{thm:z}.



\begin{theorem} \label{thm:z}
	Choose positive integer $I$, and let $\eta = 2^{-Id} $. For $\epsilon > 0$ and $t \le T$, with probability  $\geq 1 - 2 \epsilon $, for any $\delta > |\mathcal{Q}_t| \eta$ such that $\delta$ is $f$-admissible, Algorithm \ref{alg:z} (with zooming rate $\alpha$) admits $t$-episode $\delta$-regret of: 
	\begin{align} 
		\mathcal{O} \left(  \frac{1 + 2\alpha}{\alpha^2} M_d \Psi^2  \widetilde{N}_{\delta, \eta, \alpha}  
		  \log \left( \frac{T}{\epsilon}  \right) 
		 \log \log (1/\eta)  \right) , \label{eq:thm-bound-z}
	\end{align}
	where 
	$\Psi = \mathcal{O} \left( \log (T/\epsilon) + \log (1/\eta) \right)$, $\widetilde{N}_{\delta, \eta, \alpha}$ is defined in (\ref{eq:def-N-tilde}), and $\mathcal{O}$ omits constants. 
	Since each episode plays $M_d$ arms, the average $\delta$-regret each arm incurs is independent of $M_d$. 
\end{theorem} 


When proving Theorem \ref{thm:z}, the definition of $\widetilde{N}_{\delta, \eta, \alpha}$ is used in (\ref{eq:use-bound-Kv}). For a more refined bound, we can instead use
$$ \hspace{-0.25cm} \widetilde{ N }_{\delta, \eta, \alpha}'
	\hspace{-2pt}:= 
	\hspace{-2pt} 
	\sup_{ \lambda \in \left( l_{\min} , 1 \right] } \hspace{-2pt} 
	N_\delta \hspace{-2pt}\left( (1 + 2 \alpha) \log \left( M_d \lambda^d / \eta  \right) , \lambda \right) ,$$ 
where $ \l_{\min} $ is the minimal possible cube edge length during the algorithm run. This replacement will not affect the argument. Some details and an example regarding this refinement are in Appendix \ref{app:example}. 

In Remark \ref{remark:example}, we give an example of regret rate on $f (x) = 2 \log \frac{1}{x}$, $x \in (0,1]$ with specific input parameters. 

\begin{remark} \label{remark:example}
	Consider the (unbounded, BMO) function $ f(x)  = 2 \log \frac{1}{x} $, $x \hspace{-2pt} \in \hspace{-2pt} (0,1]$. Pick $T \ge 20$. For some $t \le T$, the $t$-step $\delta$-regret of Algorithm \ref{alg:z} is $\mathcal{O} \left( poly\text{-}log( t ) \right)$ while allowing $ \delta = \mathcal{O} ( 1/T ) $ 
	and $ \eta = \Theta \left( 1/T^4 \right) $. 
    Intuitively, Algorithm \ref{alg:z} gets  close to $f^\delta$ even if $f^\delta$ is very large. 
	Details of this example can be found in Appendix \ref{app:example}. 
\end{remark}

\section{Experiments}
\label{sec:exp}
We deploy Algorithms \ref{alg} and \ref{alg:z} on the Himmelblau's function and the Styblinski-Tang function (arm space normalized to $[0,1)^2$, function range rescaled to $[0,10]$). The results are in Figure \ref{fig:exp}. 
We measure performance using traditional regret and $\delta$-regret. Traditional regret can be measured because both functions are continuous, in addition to being BMO.

\begin{figure} 
	\centering
	\hspace{-8pt}\includegraphics[scale = 0.25]{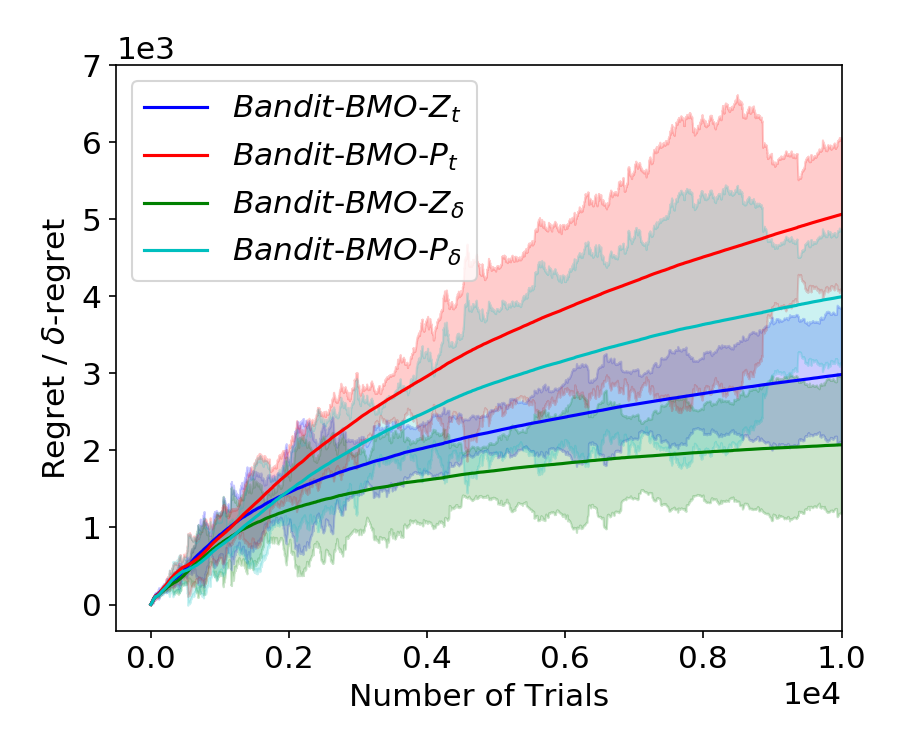} 
    \includegraphics[scale = 0.25]{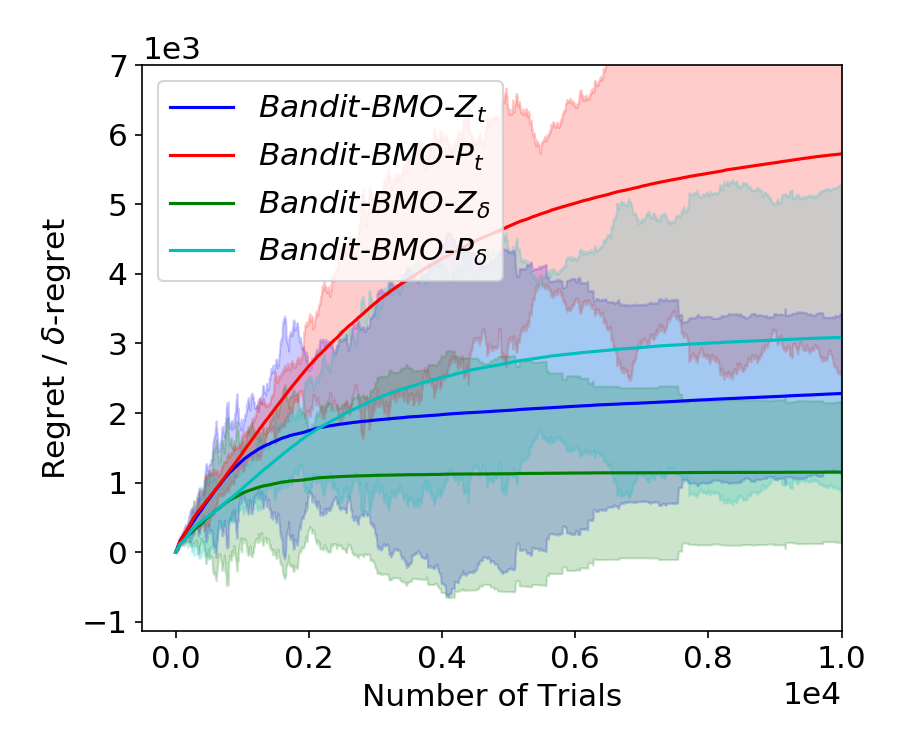}
	\caption{Algorithms \ref{alg} and \ref{alg:z} on Himmelblau's function (left) and Styblinski–Tang function (right). Each line is averaged over 10 runs. The shaded area represents one variance above and below the average regret. For the \textit{Bandit-BMO-Z} algorithm, all arms played incur regret, and each episode has 4 arm trials in it. In the figures, \textit{Bandit-BMO-Z${}_\delta$} (resp. \textit{Bandit-BMO-P${}_\delta$}) plots the $\delta$-regret ($\delta = 0.01$) for \textit{Bandit-BMO-Z} (resp. \textit{Bandit-BMO-P}). \textit{Bandit-BMO-Z${}_t$} (resp. \textit{Bandit-BMO-P${}_t$}) plots the traditional regret for \textit{Bandit-BMO-Z} (resp. \textit{Bandit-BMO-P}). For \textit{Bandit-BMO-P} algorithm, we use $\epsilon = 0.01$, $\eta = 0.001$, total number of trials $T = 10000$. For \textit{Bandit-BMO-Z} algorithm, we use $\alpha = 1$, $\epsilon = 0.01$, $\eta = 0.001$, number of episodes $T = 2500$, with four arm trials in each episode. Note that we have plotted trials (arm pulls) rather than episodes. The landscape of the test functions are in Appendix \ref{app:exp}. \label{fig:exp} }  
\end{figure}

\section{Discussion on Future Directions}
\subsection{Lower Bound}
A classic trick to derive minimax lower bounds for (stochastic) bandit problems is the ``needle-in-a-haystack.'' In this argument \citep{auer2002using}, we construct a hard problem instance, where one arm is only slightly better than the rest of the arms, making it hard to distinguish the best arm from the rest of the arms. This argument is also used in metric spaces \citep[e.g.,][]{kleinberg2008multi, lu2019optimal}. 
This argument, however, is forbidden by the definition of $\delta$-regret, since here, the set of good arms can have small measure, and will be ignored by definition. Hence, we need new insights to derive minimax lower bounds of bandit problems measured by $\delta$-regret. 

\subsection{Singularities of Analytical Forms}
In this paper, we investigate the bandit problem where the reward can have singularities in the arm space. A natural problem along this line is when the reward function has specific forms of singularities. For example, when the average reward can be written as $ f (x) = \sum_{i=1}^k \frac{1}{ (x - s_i)^{\alpha_i}} $ where $ s_i $ are the singularities and $\alpha_i$ are the ``degree'' of singularities. 
To continue leveraging the advantages of BMO function and John-Nirenberg inequalities, one might consider switching away from the Lebesgue measure and use decomposition results from classical analysis \citep[e.g.,][]{rochberg1986decomposition}. 


\section{Conclusion}
We study the bandit problem when the (expected) reward is a BMO function. We develop tools for BMO bandits, and  provide an algorithm that achieves poly-log $\delta$-regret with high probability. Our result suggests that BMO functions can be optimized (with respect to $\delta$-regret) even though they can be discontinuous and unbounded. 
\section*{Acknowledgement}

The authors thank Weicheng Ye and Jingwei Zhang for insightful discussions. The authors thank anonymous reviewers for valuable feedback.


\bibliographystyle{apa} 
\bibliography{biblio} 


\appendix
\newpage 

\newpage
\onecolumn

\section{Main Proofs}  \label{lem:main-proof}
For readability, we reiterate the lemma statements before presenting the proofs. 

\subsection{Proof of Lemma \ref{lem:delta-regret-set-measure}}

\begin{lemman} [ \ref{lem:delta-regret-set-measure} ]
    Let $f$ be the reward function. For any $f$-admissible $\delta \ge 0$, let $ S^\delta := \left\{ a \in [0,1)^d : f(a) > f^\delta \right\}$. Then we have $S^\delta$ measurable and $\mu(S^\delta) = \delta$. 
\end{lemman}

\begin{proof}[Proof]
    Recall $$F^\delta :=  \left\{ z \in \mathbb{R} :  \mu ( \{a \in [0,1)^d: f(a) > z \} ) = \delta \right\}, \qquad f = \inf F^\delta . $$ 
    We consider the following two cases.
    
    \textbf{Case 1:} $f^\delta \in F^\delta$, then by definition (of $F^\delta$), $\mu (S^\delta) = \delta$. 
    
    \textbf{Case 2:} $f^\delta \notin F^\delta $ ($F^\delta$ is left open). Then by definition of the infimum operation, for any $i = 1,2,3,\cdots$, there exists $z_i \in F^\delta $, such that $f^\delta < z_i \le f^\delta + \frac{1}{i} $. 
    Thus $\lim_{i \rightarrow \infty} z_i = f^\delta$. 
    We know that $f$ is Lebesgue measurable, since
    \begin{align*}
        f \in BMO(\mathbb{R}^d, \mu) 
        &\Rightarrow f \text{ is Lebesgue measurable. }
    \end{align*}
     Let us define $S_i := \left\{ a \in [0,1)^d: f (a) > z_i \right\}$. By this definition, $S_1\subseteq S_2 \subseteq S_3 \cdots$. Also $S_i$ is Lebesgue measurable, since it is the pre-image of the open set $(z_i, \infty)$ under the Lebesgue measurable function $f$. 
    By the above construction of $S_i$, we have $\mu(S_i) = \delta $ for all $i = 1,2,3, \cdots$.
    By continuity of measure from below, 
    \begin{align}
        \mu \left( \cup_{i=1}^\infty S_i \right) = \lim_{i \rightarrow \infty } \mu \left( S_i \right). 
    \end{align}
    We also have $S^\delta = \cup_{i=1}^\infty S_i$. This is because 
    \begin{itemize}
    	\item[(1)] $S^\delta \supseteq  \cup_{i=1}^\infty S_i$, since by definition, $S_i \subseteq S^\delta$ for all $i = 1,2,3,\cdots.$;
    	\item[(2)] $S^\delta \subseteq  \cup_{i=1}^\infty S_i$, since $\lim_{i \rightarrow \infty} z_i = f^\delta$ and therefore every element in $S^\delta$ is an element in $\cup_{i=1}^\infty S_i$. 
    \end{itemize}
    Hence, 
    \begin{align}
        \mu( S^\delta ) = \mu \left( \cup_{i=1}^\infty S_i \right) = \lim_{i \rightarrow \infty } \mu \left( S_i \right) = \delta, 
    \end{align}
    where the last equality uses $\mu \left( S_i \right) = \delta$ for all $i = 1,2,3,\cdots$. 
\end{proof}

\subsection{Lemma \ref{lem:extended-azuma} and Proof of Lemma \ref{lem:extended-azuma}}
\label{app:pf-extended-azuma}

This lemma is Proposition 34 by \citet{tao2015random}, and can be derived using Lemma 3.1 by \citet{vu2002concentration}. We prove a proof below for completeness. We will use this lemma to prove Lemma \ref{lem:bound-on-event}. 

\begin{lemma}[Proposition 34 by \citet{tao2015random}] \label{lem:extended-azuma}
	Consider a martingale sequence $ X_1, X_2,  \cdots$ adapted to filtration $\mathcal{F}_1, \mathcal{F}_2 \cdots$. For constants $c_1, c_2, \cdots < \infty$, we have 
	\begin{align}
		&\mathbb{P} \left( \left| X_n - X_0 \right| > \lambda \sqrt{ \sum_{i=1}^n c_i^2 } \right) \le 2 \exp \left( - \frac{\lambda^2}{ 2  } \right) + \sum_{i=1}^n \mathbb{P} \left( \left| X_i - X_{i-1} \right| > c_i \right).   
	\end{align}
\end{lemma}

\begin{proof}
	Define the ``good event'' $\mathcal{G}_n := 
	\left\{  
		|X_i - X_{i-1}| \le c_i, \text{ for all } i \le n
	\right\}$.  Rewrite the above probability as
	\begin{align}
		&\mathbb{P} \left( \left| X_n - X_0 \right| > \lambda \sqrt{ \sum_{i=1}^n c_i^2 } \right) \nonumber \\
		=&
		\mathbb{P} \left( \left. \left| X_n - X_0 \right| > \lambda \sqrt{ \sum_{i=1}^n c_i^2 } \right| \mathcal{G}_n \right) \mathbb{P} (\mathcal{G}_n) + 
		 \mathbb{P} \left( \left. \left| X_n - X_0 \right| > \lambda \sqrt{ \sum_{i=1}^n c_i^2 } \right| \overline{\mathcal{G}}_n \right) \left( 1-  \mathbb{P} (\mathcal{G}_n) \right) \nonumber \\
		 \le& \mathbb{P} \left( \left. \left| X_n - X_0 \right| > \lambda \sqrt{ \sum_{i=1}^n c_i^2 } \right| \mathcal{G}_n \right)  + \left( 1-  \mathbb{P} (\mathcal{G}_n) \right). \label{eq:split-event}
	\end{align}
	In (\ref{eq:split-event}), the first term can be bounded by applying Azuma's inequality for martingales of bounded difference, and the second term is the probability of there existing at least one difference being large. 
	For the first term, we define $X_i' := X_i \mathbb{I} [ |X_i - X_{i-1}| \le c_i ]$. It is clear that $\left\{ X_i' \right\}_{i} $ is also martingale sequence adapted to $\mathcal{F}_1, \mathcal{F}_2, \cdots$. Using this new sequence, we have 
	\begin{align*}
		&\mathbb{P} \left( \left. \left| X_n - X_0 \right| > \lambda \sqrt{ \sum_{i=1}^n c_i^2 } \right| {\mathcal{G}}_n \right) 
		= 
		\mathbb{P} \left(  \left| X_n' - X_0 ' \right| > \lambda \sqrt{ \sum_{i=1}^n c_i^2 }  \right) 
		\le 
		2 \exp \left( - \frac{\lambda^2}{ 2  } \right) 
		,
	\end{align*}
	where the last inequality is a direct consequence of Azuma's inequality. 
	
	Finally, we take a union bound and a complement to get $\mathbb{P} \left( \overline{\mathcal{G}}_t \right) \le \sum_{i=1}^n \mathbb{P} \left( |X_i - X_{i-1}| > c_i \right)$. This finishes the proof.
\end{proof}


\subsection{Proof of Lemma \ref{lem:bound-on-event}} 
\label{app:pf-bound-on-event} 

In order to prove Lemma \ref{lem:bound-on-event}, we need a variation of Azuma's inequality (Lemma \ref{lem:extended-azuma} in Appendix \ref{app:pf-extended-azuma}, Proposition 34 by \citet{tao2015random}). 

%

\begin{lemman}[\ref{lem:bound-on-event}] 
	Pick $T \ge 1$ and $\epsilon \in (0,1)$. With probability at least $ 1 -  \frac{\epsilon}{T}  $, the event $ \mathcal{E}_t (q) $ holds for any $q \in \mathcal{Q}_t$ at any time $t$, where 
	\begin{align*}
		\mathcal{E}_t (q) &:= 
		\begin{Bmatrix*}[l] 
		\left| \left< f \right >_{ q } - m_{t} ( q ) \right| 
		\le H_t ( q ) 
		\end{Bmatrix*} , \\
		 H_t (q) &= \BoundHoeffding{t}{q}. 
	\end{align*}
\end{lemman} 

\begin{proof}
\textbf{Case I:} We first take care of the case when $q$ contains at least one observation. 
Define $$\mathcal{F}_i' := \sigma ( Q_1, A_1, Y_1, \cdots, Q_{i-1}, A_{i-1}, Y_{i-1}, Q_i, A_i ) .$$


By our partition refinement rule, we have that for any $t, t'$ such that $t \ge t'$ and $q \in \mathcal{Q}_t$, there exists $q' \in \mathcal{Q}_{t'}$ such that $q \subseteq q'$. Thus for any $i \le t$, and any $q \in \mathcal{Q}_t$, we have either $Q_i \supseteq q $ or $Q_i \cap q = \emptyset$ ($Q_i$ is the cube played at time $i\le t$). Thus, we have 
\begin{align}
	\mathbb{E} \left[ Y_i \mathbb{I}_{ [A_i \in q] } | \mathcal{F}_{i}' \right] &= 
	\begin{cases}
		 \left< f \right>_q \mathbb{I}_{ [A_i \in q] }, \quad \text{ if } Q_i \supseteq q, \quad \text{i.e., $Q_i$ contains $q$ and the value depends on whether $q$ contains $A_i$},  \\
		 0, \quad \text{ if $Q_i \cap q = \emptyset$,} \quad \text{i.e., in this case $ \mathbb{I}_{ [A_i \in q] } = 0 $ since $A_i \in Q_i$ and $Q_i \cap q = \emptyset$, }
	\end{cases} 
	\label{eq:two-cases} \\
	&= \left< f \right>_q \mathbb{I}_{ [A_i \in q] }, \nonumber
\end{align}
where $ \mathbb{I}_{ [A_i \in q] } $ is $\mathcal{F}_i'$-measurable. In (\ref{eq:two-cases}), the two cases are exhaustive as discussed above.  

Therefore the sequence $\left\{ \left( Y_i  - \left< f \right>_q \right)\mathbb{I}_{ [A_i \in q] } \right\}_i$ is a (skipped) martingale difference sequence adapted to $\F_i'$, with the skipping event $ \mathbb{I}_{ [A_i \in q] } $ being $\mathcal{F}_i'$-measurable. 

Let $ A' $ be a uniform random variable drawn from the cube $q$. We have 
\begin{align}
	\mathbb{P} \left( \left| \left( f(A_i)  - \left< f \right>_q \right) \mathbb{I}_{ [A_i \in q ] } \right| > \Psi  \right) &= 
	\begin{cases}
		\mathbb{P} \left( \left| f(A')  - \left< f \right>_q  \right| > \Psi \right), \quad \text{ if } A_i \in q, \\
		0, \quad \text{ otherwise }. 
	\end{cases} \nonumber \\
	&\le \mathbb{P} \left( \left| f(A')  - \left< f \right>_q  \right| > \Psi \right) \nonumber \\
	&= \frac{ \mu \left( a \in q : \left| \left( f(a)  - \left< f \right>_q \right) \right| > \Psi \right) }{ \mu (q) } \nonumber \\
	&\le \frac{ \mu(q)  \exp \left( - {\Psi } \right)}{ \mu (q) } \le \frac{ \epsilon }{T^2}, \label{eq:use-JN} 
\end{align} 
where (\ref{eq:use-JN}) is from the John-Nirenberg inequality. 

Next, since 
\begin{align*}
    \left| \left( Y_i  - \left< f \right>_q \right) \mathbb{I}_{ [A_i \in q ] } \right| \le \left| \left( f(A_i)  - \left< f \right>_q \right) \mathbb{I}_{ [A_i \in q ] } \right| + \left| \left( Y_i  - f(A_i) \right) \mathbb{I}_{ [A_i \in q ] } \right|, 
\end{align*}
we have 
\begin{align}
    \mathbb{P} \left( \left| \left( Y_i  - \left< f \right>_q \right) \mathbb{I}_{ [A_i \in q ] } \right| \le \Psi + D_{\mathpzc{E}} \right) 
    &\ge \mathbb{P} \left(  \left| \left( f(A_i)  - \left< f \right>_q \right) \mathbb{I}_{ [A_i \in q ] } \right| + \left| \left( Y_i  - f(A_i) \right) \mathbb{I}_{ [A_i \in q ] } \right| \le \Psi + D_{\mathpzc{E}} \right) \nonumber \\ 
    &\ge \mathbb{P} \left(  \left| \left( f(A_i)  - \left< f \right>_q \right) \mathbb{I}_{ [A_i \in q ] } \right| \le \Psi \;\; \text{ and } \;\; \left| \left( Y_i  - f(A_i) \right) \mathbb{I}_{ [A_i \in q ] } \right| \le  D_{\mathpzc{E}} \right) \nonumber \\
    &= 1 - \mathbb{P} \left(  \left| \left( f(A_i)  - \left< f \right>_q \right) \mathbb{I}_{ [A_i \in q ] } \right| > \Psi \;\; \text{ or } \;\; \left| \left( Y_i  - f(A_i) \right) \mathbb{I}_{ [A_i \in q ] } \right| >  D_{\mathpzc{E}} \right) \nonumber \\
    &\ge 1 - \mathbb{P} \left(  \left| \left( f(A_i)  - \left< f \right>_q \right) \mathbb{I}_{ [A_i \in q ] } \right| > \Psi \right) - \mathbb{P} \left( \left| \left( Y_i  - f(A_i) \right) \mathbb{I}_{ [A_i \in q ] } \right| >  D_{\mathpzc{E}} \right), \label{eq:use-union}
\end{align}
where (\ref{eq:use-union}) uses a union bound. 

By a union bound and the John-Nirenberg inequality, for any $i \le t$, and $q \in \mathcal{Q}_t$, we have
	\begin{align}
	    \mathbb{P} \left( \left| \left( Y_i  - \left< f \right>_q \right) \mathbb{I}_{ [A_i \in q ] } \right| > \Psi + D_{\mathpzc{E}} \right) 
		&= 1 - 
		\mathbb{P} \left( \left| \left( Y_i  - \left< f \right>_q \right) \mathbb{I}_{ [A_i \in q ] } \right| \le \Psi + D_{\mathpzc{E}} \right) \nonumber \\
		&\le
		\mathbb{P} \left( \left| \left( f(A_i)  - \left< f \right>_q \right) \mathbb{I}_{ [A_i \in q ] } \right| > \Psi  \right) + \mathbb{P} \left( \left| \left( f(A_i)  - Y_i \right) \mathbb{I}_{ [A_i \in q ] } \right| >  D_{\mathpzc{E}} \right) \label{eq:use-use-union}
		\\
		&= \mathbb{P} \left( \left| \left( f(A_i)  - \left< f \right>_q \right) \mathbb{I}_{ [A_i \in q ] } \right| > \Psi  \right) \label{eq:use-bounded-noise} \\
		&\le \frac{ \epsilon }{T^2}, \label{eq:use-use-JN}
	\end{align} 
	where (\ref{eq:use-use-union}) uses (\ref{eq:use-union}), (\ref{eq:use-bounded-noise}) uses the boundedness of noise (\textbf{N1}), and (\ref{eq:use-use-JN}) uses (\ref{eq:use-JN}).  
	
%
	To put it all together, we can apply Lemma \ref{lem:extended-azuma} to the (skipped) martingale $\left\{ \sum_{j=1}^i \left( Y_j  - \left< f \right>_q \right)\mathbb{I}_{ [A_j \in q] } \right\}_{i=1,2,\cdots}$ (with $c_i = \Psi + D_{\mathpzc{E}}$, $\lambda = \sqrt{2 \log (2T^2/\epsilon)}$, and $X_i = \left( Y_i  - \left< f \right>_q \right)\mathbb{I}_{ [A_i \in q] } $)
	to get for $T \ge 2$ and a cube $q \in \mathcal{Q}_t$ such that $n_t (q) > 0$, 
	\begin{align}
		&\mathbb{P} \left(  \left| 
			\sum_{i=1}^{t-1} \left( Y_i - \left< f \right>_q \right)\mathbb{I}_{ [A_i \in q] } 
		\right|  > (\Psi + D_{\mathpzc{E}})  \sqrt{ n_t (q) }  \sqrt{ \HoeffTSq }
		\right) \label{eq:for-skipping} \\
		\le & 
		2 \exp \left( - \frac{ \HoeffTSq }{2} \right) + 
		\sum_{i=1}^{t-1} \mathbb{P} \left( \left|
		\left( Y_i  - 
		\left< f \right>_q \right) 
				\mathbb{I}_{ [A_i \in q] }
			\right| > \Psi + D_{\mathpzc{E}} \right) \label{eq:use-bound-52} \\
		\le& \frac{ \epsilon }{ T^2 } + (t-1) \frac{ \epsilon }{T^2} \le \frac{\epsilon }{T}, \nonumber
	\end{align}
where (\ref{eq:for-skipping}) uses Lemma \ref{lem:extended-azuma}, (\ref{eq:use-bound-52}) uses (\ref{eq:use-use-JN}) for the summation term. 

Since $n_t(q) > 0$, we use 
\begin{align*} 
	m_t (q)  = \frac{1}{n_t (q)} \sum_{i=1}^{t-1} Y_i \mathbb{I}_{ [A_i \in q]}, 
\end{align*} 
to rewrite (\ref{eq:for-skipping}) by dividing both sides by $ n_t (q) $ to get
\begin{align*} 
	&\mathbb{P} \left(  \left| 
			m_t (q) - \left< f \right>_q
		\right|  > \frac{ (\Psi + D_{\mathpzc{E}}) }{\sqrt{n_t (q)}} \sqrt{ \HoeffTSq } 
		\right) \le \frac{\epsilon}{T} . 
\end{align*} 

\textbf{Case II:} Next, we consider the case where $q$ contains no observations. 

In order to do this, we need Propositions \ref{prop:cube-number} and \ref{prop:bound-on-consecutive-mean}, which are proved in \ref{app:pf-cube-number} and \ref{app:pf-prop-consecutive-mean}. 

\begin{proposition} \label{prop:cube-number}	Following (\ref{eq:partition-rule}), the minimal cube measure is at least ${\eta}$. Thus the maximal number of cubes produced by Algorithm \ref{alg} is $\frac{1}{\eta}$, since the arm space is of measure 1. 
\end{proposition} 


\begin{proposition} \label{prop:bound-on-consecutive-mean}
	For a function $f \in BMO(\mathbb{R}^d, \mu)$, and rectangles $ q_{{}_0}, q_{{}_1}, \cdots, q_{{}_k} $ such that $ q_{{}_0} \subseteq q_{{}_1} \subseteq q_{{}_2} \subseteq \cdots \subseteq q_{{}_k}, $ and constant $K \ge 1$ such that $ K \mu ( q_{{{}_i}} )  \ge \mu ( q_{{}_{i+1}} ) $ for all $i \in  [0,k-1]$, we have 
	\begin{align*}
		\left| \left< f \right>_{q_{{}_0}} - \left< f \right>_{q_{{}_k} } \right| \le K k \left\| f \right\|. 
	\end{align*}
\end{proposition}

Let's continue with the proof of Lemma \ref{lem:bound-on-event}.
By the lower bound on cube measure (Proposition \ref{prop:cube-number}), we know that $ \mu (q) \ge { \eta} $ for any $q$ generated by the algorithm. 
Let us construct a sequence of hyper-rectangles $ q = q_{{}_0}, q_{{}_1}, \cdots, q_{{}_k} \subseteq [0,1)^d $, 
such that $ q_{{}_i} \subseteq q_{{}_{i+1}} $ for $ i = 0,1,\cdots,k $, $ \mu(q_{{}_{i+1}}  ) = 2 \mu(q_{{}_i} ) $, and $ q_{{}_{k}} = [0,1)^d $. Since $ q $ is generated by the algorithm, we know $ \mu (q) \ge \eta $ (Proposition \ref{prop:cube-number}). For this sequence of hyper-rectangles, $k \le \log_2 ( 1 / \eta ) $. 
	
Then by Proposition \ref{prop:bound-on-consecutive-mean}, 
	\begin{align}
		\left| \left< f \right>_{q} -  \left< f \right>_{[0,1)^d} \right| 
		\le 2 \log_2 (1/\eta)  \left\| f \right\|. \label{eq:lem-con-no-hit-dummy}
	\end{align}
	Thus by definition of the functions $m_t$, $n_t$ for cubes with no observations, for a cube $q$ such that $n_{t}(q) = 0$, 
	\begin{align*}
		\left| \left< f \right>_q - m_{t}(q) \right| &\overset{ \text{\textcircled{1}} }{=} \left| \left< f \right>_q \right| 
		\overset{ \text{\textcircled{2}} }{=} \left| \left< f \right>_q - \left< f \right>_{[0,1)^d} \right| \overset{ \text{\textcircled{3}} }{\le} 2 \log_2 (1/\eta) \left\| f \right\| \overset{ \text{\textcircled{4}} }{\le } \frac{\Psi }{\sqrt{ \max (1, n_{t} (q) ) } } \overset{ \text{\textcircled{5}} }{\le } \frac{ (\Psi + D_{\mathpzc{E}} ) \sqrt{ 2\log (2T^2 / \epsilon) } }{\sqrt{ \max (1, n_{t} (q) ) } } ,
	\end{align*}
	where \text{\textcircled{1}} is due to $ m_t(q) = 0 $ when $n_t ( q )  = 0$ by definition, \text{\textcircled{2}} is from Assumption \ref{assumption:mean-zero} \big($\left< f\right>_{[0,1)^d} = 0$\big), \text{\textcircled{3}} is from (\ref{eq:lem-con-no-hit-dummy}), and \textcircled{4} is from $ 2\log_2 (1/ \eta) \le \Psi$ (Eq. \ref{eq:def-Psi}) and $n_t(q) = 0$. Recall we assume $\| f \| = 1$ for cleaner representation. 
We have finished the proof of Lemma \ref{lem:bound-on-event}.
\end{proof}

\subsection{Proof of Proposition \ref{prop:cube-number}}
\label{app:pf-cube-number}

\begin{propositionn}[\ref{prop:cube-number}]
	Following (\ref{eq:partition-rule}), the maximal number of cubes produces by Algorithm \ref{alg} is $\frac{1}{\eta}$. The minimal cube measure is at least ${\eta}$. 
\end{propositionn}

\begin{proof}
	This proposition is an immediate consequence of our partition refinement rule (\ref{eq:partition-rule}). The cube measures cannot be smaller than $ {\eta} $. Otherwise, the RHS of the rule (\ref{eq:partition-rule}) will be nonpositive and no more splits will happen. 
\end{proof}


\subsection{Proof of Proposition \ref{prop:bound-on-consecutive-mean}}
\label{app:pf-prop-consecutive-mean} 

Proposition \ref{prop:bound-on-consecutive-mean} is a property of BMO functions, and can be found in textbooks \citep[e.g.,][]{stein1993harmonic}. 

\begin{propositionn}[\ref{prop:bound-on-consecutive-mean}] 
	For a function $f \in BMO(\mathbb{R}^d, \mu)$, and rectangles $ q_{{}_0}, q_{{}_1}, \cdots, q_{{}_k} $ such that $ q_{{}_0} \subseteq q_{{}_1} \subseteq q_{{}_2} \subseteq \cdots \subseteq q_{{}_k}, $ and a constant $K \ge 1$ such that $ K \mu ( q_{i} )  \ge \mu ( q_{i+1} ) $ for all $i \in  [0,k-1]$, we have 
	\begin{align*}
		\left| \left< f \right>_{q_{{}_0}} - \left< f \right>_{q_{{}_k} } \right| \le K k \left\| f \right\|. 
	\end{align*}
\end{propositionn}

\begin{proof}
	The proof is a consequence of basic properties of BMO function. For any two regular rectangles $q_{{}_i}$ and $q_{{}_{i+1}} $ ($i = 0,1,2, \cdots, k-1$), 
	\begin{align}
		\left| \left< f \right>_{q_{{}_i}} - \left< f \right>_{q_{{}_{i+1}} } \right|  
		&= 
		\left| \frac{1}{\mu(q_{{}_i})} \int_{q_{{}_i}} f d \mu  - \left< f \right>_{ q_{{}_{i+1}} } \right| \nonumber \\
		&= 
		\left| \frac{1}{\mu(q_{{}_i})} \int_{q_{{}_i}} \left( f - \left< f \right>_{ q_{{}_{i+1}} } \right) d \mu   \right| \nonumber \\
		&\le \nonumber 
		 \frac{1}{\mu(q_{{}_i})} \int_{q_{{}_i}} \left|  f - \left< f \right>_{ q_{{}_{i+1}} }   \right| d \mu \\
		&\le  \frac{K}{\mu( q_{{}_{i+1}} )} \int_{q_{{}_i}} \left|  f - \left< f \right>_{q_{{}_{i+1}}}   \right| d \mu 		 \label{eq:bmo-prop-dummy1} \\ 
		&\le \frac{K}{\mu( q_{{}_{i+1}} )} \int_{q_{{}_{i+1}}} \left|  f - \left< f \right>_{q_{{}_{i+1}}}   \right| d \mu  \label{eq:bmo-prop-dummy2}\\
		&\le  K \left\| f \right\| , \nonumber
	\end{align} 
	where (\ref{eq:bmo-prop-dummy1}) uses $ K \mu(q_i) \ge \mu (q_{{}_{i+1}}) $ and (\ref{eq:bmo-prop-dummy2}) uses $ q_{{}_i} \subseteq q_{{}_{i+1}}$. 
	Next, we use the triangle inequality and repeat the above inequality $k$ times to get
	\begin{align*}
		\left| \left< f \right>_{q_{{}_0}} - \left< f \right>_{q_{{}_k} } \right| \le \sum_{i=1}^k \left| \left< f \right>_{q_{i-1}} - \left< f \right>_{q_{i}} \right| \le K k \left\| f \right\|. 
	\end{align*}
\end{proof}
%
%
%

\subsection{Proof of Lemma \ref{lem:exist-good-cube-arm}}

\begin{lemman}[\ref{lem:exist-good-cube-arm}]
    For any partition $\mathcal{Q}$ of $[0,1)^d$, there exists $q \in \mathcal{Q}$, such that 
    \begin{align} 
        f^\delta \le  \AverageValue{q} +
        \BoundMeasurecube{q}, \label{eq:con-measure-max-in-app}
    \end{align}
    for any $f$-admissible $\delta > \eta |\mathcal{Q}|$, where $|\mathcal{Q}|$ is the cardinality of $Q$.  
\end{lemman} 

\begin{proof}

We use $f^\delta$ and $S^\delta$ as in Lemma \ref{lem:delta-regret-set-measure}. 

Suppose, in order to get a contradiction, that for every cube $q \in \mathcal{Q}$, (\ref{eq:con-measure-max-in-app}) is violated. 

Define 
\begin{align*} 
    &S(q) := \left\{ a \in q : f (a) > \AverageValue{q} +
        \BoundMeasurecube{q} \right\} ,  \\  
    & \tilde{S}(q) := \left\{  a \in q : f (a) > f^\delta \right\}. 
\end{align*} 
Suppose the lemma statement is false. For all $q \in \mathcal{Q}$, $f^\delta >  \AverageValue{q} + \BoundMeasurecube{q} $. Thus we have 
for all $q \in \mathcal{Q}$, 
\begin{align*}
    \tilde{S}(q) \subseteq S(q). 
\end{align*} 
We have, by the John-Nirenberg inequality,
\begin{align*}
    \mu (S(q)) \le \mu \left( \left\{ a \in q : | f (a) - \left< f \right>_q  | > \log (\mu (q) / \eta) \right\} \right) \le  \eta. 
\end{align*}
Since $\mathcal{Q}$ is a partition (of $[0,1)^d$), we have 
\begin{align*}
    \mu ( \cup_{ q \in \mathcal{Q} } {S}(q)) = \sum_{q \in \mathcal{Q} }\mu (S(q)) \le \sum_{q \in \mathcal{Q}} \eta = |\mathcal{Q}| \eta.  
\end{align*} 
On the other hand, by definition of $f^\delta$ and disjointness of the sets $\tilde{S} (q)$, we have 
\begin{align*}
    &\mu ( \cup_{ q \in \mathcal{Q} } \tilde{S}(q)) 
    = \mu \left( S^\delta \right) = \delta.
\end{align*} 
Since $\delta > |\mathcal{Q}| \eta $, we have
\begin{align*}
    \mu ( \cup_{ q \in \mathcal{Q} } \tilde{S}(q)) > \mu ( \cup_{ q \in \mathcal{Q} } {S}(q)), 
\end{align*} 
which is a contradiction to $  \tilde{S}(q) \subset S(q) $ for all $q$. This finishes the proof. 
\end{proof}

\subsection{Proof of Theorem \ref{thm:regret-bound}}
\label{app:pf-thm-p}
\begin{theoremn}[\ref{thm:regret-bound}]
    
    Fix any $T$. With probability at least $ 1 - {2 \epsilon} $ , for any $\delta > |\mathcal{Q}_T | { \eta}$ such that $\delta$ is $f$-admissible, the total $\delta$-regret for Algorithm \ref{alg} up to time $T$ is 
    \begin{align}
        \sum_{t=1}^T r_{t}^\delta  
        \le \widetilde{\mathcal{O}} \left( \sqrt{ T | \mathcal{Q}_T | } \right)  \label{eq:use-point-scattering-in-app} 
         ,
    \end{align}
    where $\mathcal{Q}_T$ is the cardinality of $\mathcal{Q}_T$. 
\end{theoremn}

\begin{proof}
	Under the ``good event'' $ \mathcal{E}^{good} := \left( \bigcap_{t=1}^T \mathcal{E} (Q_t) \right) \cap \left( \bigcap_{t=1}^T \mathcal{E} (q_t^{\max}) \right) , $ we continue from (\ref{eq:get-single-regret}) and get 
\begin{align}
    \sum_{t =1}^T r_{t}^\delta \le& \sum_{t = 1 }^T 3 \BoundHoeffding{t}{Q_{t}} \label{eq:use-23} \\
    \le& 3   \left( \Psi + D_{\mathpzc{E}} \right) \sqrt{ \HoeffTSq }  \sqrt{T} \cdot \sqrt{ \sum_{t=1}^T \frac{ 1 }{ \max( 1, n_{t-1} (Q_{t }) ) } }   \label{eq:use-cauchy} \\ 
    \le& 3 \left( \Psi + D_{\mathpzc{E}} \right) \sqrt{ \HoeffTSq } \sqrt{T} \cdot \sqrt{ e | \mathcal{Q}_T |  \log \left( 1 + (e - 1) \frac{ T }{  | \mathcal{Q}_T |  } \right) } \label{eq:use-point-scattering-2} 
\end{align} 
where (\ref{eq:use-23}) uses (\ref{eq:get-single-regret}), where (\ref{eq:use-cauchy}) uses the Cauchy-Schwarz inequality, (\ref{eq:use-point-scattering-2}) uses (\ref{eq:point-scattering-gp}). 

What remains is to determine the probability under which the ``good event'' happens. By Lemma \ref{lem:bound-on-event} and a union bound, we know that the event $\mathcal{E}^{good}$ happens with probability at least $1 - 2 \epsilon $. 
\end{proof}

\subsection{Proof of Proposition \ref{prop:parent-partition}}

\begin{propositionn} [\ref{prop:parent-partition}]
	At any episode $t$, the collection of parent cubes forms a partition of the arm space. 
\end{propositionn} 
\begin{proof}
	We first argue that any two parent cubes do not overlap. 
	By definition, all parent cubes are dyadic cubes. By definition of dyadic cubes (\ref{eq:dyadic-cubes}), two different dyadic cubes $Q$ and $Q'$ such that $ Q \cap Q' \neq \emptyset $ must satisfy either (i) $ Q' \subseteq Q $ or (ii) $ Q \subseteq Q' $. From the definition of parent cubes and pre-parent cubes, we know a parent cube cannot contain another parent cube. Thus for two parent cubes $Q$ and $Q'$, $ Q \cap Q' \neq \emptyset $ implies $ Q = Q'$. Thus two different parent cubes cannot overlap. 
	
	We then argue that the union of all parent cubes is the whole arm space. 
	We consider the following cases for this argument. Consider any pre-parent cube $Q$. (1) If $Q$ is already a parent cube, then it is obviously contained in a parent cube (itself). (2) At time episode $t$, if $Q \in \mathcal{Q}_t$ is a pre-parent cube but not a parent cube, then by definition it is contained in another pre-parent cube $Q_1$. If $Q_1$ is a parent cube, then $Q$ is contained in a parent cube. If $Q_1$ is not a parent cube yet, then $Q_1$ is contained in another pre-parent cube $Q_2$. We repeat this argument until we reach $[0,1)^d$ which is a parent cube as long as it is a pre-parent cube. For the boundary case when $[0,1)^d$ is a terminal cube, it is also a parent cube by convention. Therefore, any pre-parent cube is contained in a parent cube. 
	
	Next, by definition of pre-parent cubes and the zooming rule, any terminal cube is contained in a pre-parent cube. Thus any terminal cube is contained in a parent cube. 
	
	
	Since terminal cubes cover the arm space by definition, the parent cubes cover the whole arm space. 
\end{proof}

\subsection{Proof of Proposition \ref{prop:regroup}} \label{app:pf-prop-regroup}
\begin{propositionn}[\ref{prop:regroup}]
	Following the Zooming Rule (\ref{eq:split-rule-zooming}), we have  
	\begin{itemize}
	\item[1.] Each parent cube of measure $\mu$ is played at most $ \maxHitNumber{ \mu / \eta  } $ episodes. \\
	\item[2.] Under event $\widetilde{\mathcal{E}}_T$, each parent cube $Q_t$ selected at episode $t$ is a subset of $\mathcal{X}_\delta \left( (1 + 2 \alpha)  \log \left( M_d \mu (Q_t) / \eta \right)  \right)$. 
	\end{itemize}
\end{propositionn} 

\begin{proof}
	For item 1, every time a parent cube $Q$ of measure $\mu$ is selected, all $M_d$ of its direct sub-cubes are played. The direct sub-cubes are of measure $\frac{\mu}{M_d}$, and each such cube can be played at most $ \maxHitNumber{ \frac{  \mu }{ \eta } }  $ times. Beyond this number, rule (\ref{eq:split-rule-zooming}) will be violated, and all the direct sub-cubes can no longer be terminal cubes. Thus $Q$ will no longer be a parent cube (since $Q$ is no longer a pre-parent cube), and is no longer played. 
	
	Item 2 is a rephrasing of (\ref{eq:bound-Delta}). 
	Assume that event $\widetilde{\mathcal{E}}_T = \left( \bigcap_{t = 1}^T {\mathcal{E}}_{t} ({q}_{t}^{\max} ) \right) \bigcap \left( \bigcap_{t = 1}^T {\mathcal{E}}_{t} (Q_{t} ) \right)$ is true. 
	Let $Q_t$ be the parent cube for episode $t$. By (\ref{eq:bound-f-delta-JH}), we know, under event $\widetilde{\mathcal{E}}_T$, there exists a ``good'' parent cube $ {q}_t^{\max} $ such that 
	\begin{align*} 
		f^\delta \le m_t ( {q}_t^{\max} ) + H_t ( {q}_t^{\max} ) + J ( {q}_t^{\max} ). 
	\end{align*} 
	By the concentration result in Lemma \ref{lem:bound-on-event}, we have, under event $\widetilde{\mathcal{E}}_T$,
	\begin{align*}
    	\left< f \right>_{Q_t} \ge m_t ( Q_t ) - H_t ( Q_t ). 
	\end{align*} 
	Combining the above two inequalities gives 
	\begin{align}
    	f^\delta - \left< f \right>_{Q_t} &\le m_t ( {q}_t^{\max} ) + H_t ( {q}_t^{\max} ) + J ( {q}_t^{\max} ) - m_t ( Q_t ) + H_t ( Q_t ) \nonumber \\
    	&\le m_t ( Q_t ) + H_t ( Q_t ) + J ( Q_t ) - m_t ( Q_t ) + H_t ( Q_t ) \label{eq:use-opt-app} \\
    	&\le J ( Q_t ) + 2 H_t ( Q_t ) \nonumber \\
    	&\le (1 + 2\alpha) \log ( M_d \mu (Q_t) / \eta), \label{eq:use-zooming-rule-app}
	\end{align} 
	where (\ref{eq:use-opt-app}) uses $ U_t (Q_t) \ge U_t ({q}_t^{\max}) $ by optimistic nature of the algorithm, and (\ref{eq:use-zooming-rule-app}) uses rule (\ref{eq:split-rule-zooming-parent}). 
\end{proof}

\subsection{Elaboration of Remark \ref{remark:example}} \label{app:example}

\begin{remarkn}[\ref{remark:example}]
	Consider the (unbounded, BMO) function $ f(x)  = 2 \log \frac{1}{x} $, $x \hspace{-2pt} \in \hspace{-2pt} (0,1]$. Pick $T \ge 20$. For some $t \le T$, the $t$-step $\delta$-regret of Algorithm \ref{alg:z} is $\mathcal{O} \left( poly\text{-}log( t ) \right)$ while allowing $ \delta = \mathcal{O} ( 1/T ) $ 
	and $ \eta = \Theta \left( 1/T^4 \right) $. 
    Intuitively, Algorithm \ref{alg:z} gets  close to $f^\delta$ even if $f^\delta$ is very large. 
\end{remarkn} 


    Firstly, recall the zooming number is defined as 
    \begin{align}
        \widetilde{ N }_{\delta, \eta, \alpha} := 
	    \sup_{ \lambda \in \left(  \eta^{\frac{1}{d}}, 1 \right] }  \hspace{-2pt} 
	    N_\delta \hspace{-2pt}\left( (1 + 2 \alpha) \log  \left( M_d \lambda^d / \eta  \right) , \lambda \right). 
    \end{align}
    While this number provide a regret bound, it might overkill by allowing $\lambda$ to be too small. We define a refined zooming number 
    \begin{align}
        \widetilde{ N }_{\delta, \eta, \alpha}' := 
	    \sup_{ \lambda \in \left(  l_{\min}, 1 \right] }  \hspace{-2pt} 
	    N_\delta \hspace{-2pt}\left( (1 + 2 \alpha) \log  \left( M_d \lambda^d / \eta  \right) , \lambda \right),
    \end{align}
    where $l_{\min}$ is the minimal possible cube edge length during the algorithm run. We will use this refined zooming number in this example. 
    Before proceeding, we put forward the following claim.

    \textbf{Claim. } Following rule (\ref{eq:split-rule-zooming-parent}), the minimal cube measure $\mu_{\min}$ at time $T$ is at least $ \Omega \left( 2^{ - \frac{ \Psi  \sqrt{2 \log (2T^2 / \epsilon)} }{ \alpha \log 2 } } \right) $. 
    \begin{proof}[Proof of Claim]
        In order to reach the minimal possible measure, we consider keep playing the cube with minimal measure (and always play a fixed cube if there are ties) and follow rule (\ref{eq:split-rule-zooming-parent}). Let $ t_i $ be the episode where $i$-th split happens. Since we keep playing the cube with minimal measure, 
        \begin{align*}
            \frac{ (\Psi + D_{\mathpzc{E}} ) \sqrt{2 \log (2T^2 / \epsilon)} }{ t_i } \approx_d \alpha \BoundMeasure{ \frac{ M_d 2^{-di} }{\eta} }. 
        \end{align*}
        By taking difference between consecutive terms, 
        \begin{align*}
            \frac{ (\Psi + D_{\mathpzc{E}} ) \sqrt{2 \log (2T^2 / \epsilon)} }{ t_i }
            - 
            \frac{ (\Psi + D_{\mathpzc{E}} ) \sqrt{2 \log (2T^2 / \epsilon)} }{ t_{i+1} } 
            \approx_d \alpha \log M_d , 
        \end{align*}
        where $\approx_d$ omits dependence on $d$. 
        
        Let $i_{\max}$ be the maximal number of splits for $T$ episodes. By using $t_0 = 1$ and $t_{i_{\max}} \le T$, 
        the above approximate equation gives
        \begin{align}
            \sum_{i=0}^{i_{\max}} \left( \frac{ (\Psi + D_{\mathpzc{E}} ) \sqrt{2 \log (2T^2 / \epsilon)} }{ t_i }
            - 
            \frac{ (\Psi + D_{\mathpzc{E}} ) \sqrt{2 \log (2T^2 / \epsilon)} }{ t_{i+1} } \right)
            \approx_d  i_{\max} \alpha \log M_d \\
            (\Psi + D_{\mathpzc{E}} ) \sqrt{2 \log (2T^2 / \epsilon)} 
            \gtrsim_d  i_{\max} \alpha \log M_d,  
        \end{align} 
        where the approximations omit possible dependence on $d$. 
        This gives, by using $M_d = 2^d$, 
        \begin{align}
            i_{\max} \lesssim_d \frac{ (\Psi + D_{\mathpzc{E}} ) \sqrt{2 \log (2T^2 / \epsilon)} }{ \alpha \log M_d } \le \frac{ (\Psi + D_{\mathpzc{E}} ) \sqrt{2 \log (2T^2 / \epsilon)} }{ \alpha d \log 2 }.
        \end{align} 
        Since each split decrease the minimal cube measure by a factor of $ M_d = 2^d $, we have 
        \begin{align}
            \mu_{\min} \gtrsim 2^{-d i_{\max}} \gtrsim 2^{ - \frac{ (\Psi + D_{\mathpzc{E}} ) \sqrt{2 \log (2T^2 / \epsilon)} }{ \alpha \log 2 } } . 
        \end{align}
        \textit{Now we finished the proof of the claim. }
    \end{proof}
    
	Consider the function $f(x) = 2 \log \frac{1}{x} $, $x \hspace{-2pt} \in \hspace{-2pt} (0,1]$. 
	
	Recall
	\begin{align*} 
		\mathcal{X}_{\delta} (\lambda) := \left\{ q \subseteq (0,1]: \left< f \right>_q \ge f^\delta - \lambda \right\}. 
	\end{align*}
	For this elementary decreasing function $f (x)$, we have $f^\delta = 2 \log \frac{1}{\delta}$, and $\left< f \right>_{(0,x]} = 2 + 2 \log \frac{1}{x}$ for $x \in (0,1]$. Thus, 
	\begin{align*}
		\mathcal{X}_{\delta} (\lambda) = \left\{ x\in (0,1]: \log x \le 1 + \frac{\lambda}{2} + \log \delta \right\}. 
	\end{align*}
	By a substitution of $ \lambda \leftarrow (1 + 2 \alpha) \log ( M_d \lambda^d / \eta) $, and using $d = 1$ and $M_d = 2$, we have 
	\begin{align} 
		\mathcal{X}_{\delta} ( (1 + 2\alpha)  \log ( 2 \lambda / \eta)) = \left\{ x \in (0,1]: x \le  e \left( \frac{ 2 \lambda }{\eta} \right)^{\frac{1 + 2 \alpha}{2} } \delta  \right\}.   \label{eq:chi-delta}
	\end{align} 
	Consider the first $t$ ($t \le T$) step $\delta$-regret. 
	For simplicity, let $t = T^\beta$ for some $\beta < 1$, $|\mathcal{Q}_t| = t$, $\eta = \frac{1}{ T^{4} }$ and $ \delta = \frac{2 t}{T^4} = 2 T^{\beta - 4} $. We can do this since any $\delta >0$ is $f$-admissible. Next we will study the zooming number $\widetilde{N}_{\delta,\eta,\alpha}$ under this setting. 
	Back to (\ref{eq:chi-delta}) with the above numbers, 
	\begin{align*} 
		\mathcal{X}_{\delta} ( (1 + 2 \alpha) \log ( 2 \lambda / \eta)) = \left\{ x \in (0,1]: x \le 2^{ \frac{ 3 + 2 \alpha }{2} } e \cdot   \lambda^{\frac{1 + 2 \alpha}{2}} T^{4 \alpha + \beta - 2} \right\}. 
	\end{align*} 
	As an example, we take $\alpha = \frac{1}{4}$ and $\beta = \frac{1}{2}$, which gives $ \mathcal{X}_{\delta} ( (1 + 2 \alpha) \log ( 2 \lambda / \eta)) = \left\{ x \in (0,1]: x \le 2^{ 7/4 } e  \lambda^{3/4} T^{-1/2} \right\} $. By the choice of $ (\delta, \eta, \alpha) $ and the claim above, for $T $ large enough ($ T \ge 20 $ is sufficient), we have 
	\begin{align}
	    \mu_{\min} &\gtrsim 2^{ - \frac{ \Psi  \sqrt{2 \log (2T^2 / \epsilon)} }{ \alpha \log 2 } } \nonumber \\
	    &\gtrsim 2^{ - (\log T)^2 } \nonumber \\
	    &\gtrsim T^{-2}, \label{eq:min-cube-measure-app}
	\end{align}
	where the last step uses $ T^{-2} \le 2^{ - (\log T)^2 } $ for $ T \ge 20 $. To bound $\widetilde{N}_{\delta,\eta,\alpha}'$, we consider the following two cases.
	
	\textbf{Case I:} $ 2^{{7}/{4}} e \cdot   \lambda^{3/4} T^{-1/2} \le 1 $, i.e., $\lambda \lesssim T^{2/3}$. In this case, we need to use intervals of length $ \lambda $ to cover $\left( 0, 2^{{7}/{4}} e \cdot \lambda^{3/4} T^{-1/2} \right]$. We need $ \mathcal{O} \left(  \lambda^{-1/4} T^{-1/2} \right) $ intervals to cover it, which is at most $ \mathcal{O} \left( 1 \right) $, since $\lambda \gtrsim T^{-2} $ by (\ref{eq:min-cube-measure-app}). 
	
	\textbf{Case II:} $ 2^{{7}/{4}} e \cdot   \lambda^{\frac{3}{4}} T^{-\frac{1}{2}} > 1 $, i.e., $\lambda \gtrsim  T^{2/3} $. In this case, we need to use intervals of length $ \lambda $ to cover $\left( 0, 1 \right]$. We need $ \mathcal{O} \left( \lambda^{ - 1 } \right) $ intervals to cover it, which is at most $ \mathcal{O} \left( 1 \right) $, since $\lambda \gtrsim  T^{2/3} \ge 1 $. 
	
	
	In either case, we have $\widetilde{N}_{\delta,\eta,\alpha}' = \mathcal{O} (1)$. Plugging back into Theorem \ref{thm:z} gives, with high probability, for the first $ t = \sqrt{T} $ steps, the $\delta$-regret ($\delta = \mathcal{O} (1/T)$) is of order $ poly\text{-}log( T) $, which is $ poly\text{-}log( t) $ since $T = t^2$. 

\section{Additional Proof: Proof of Lemma \ref{lem:point-scattering}}
\label{app:point-scattering-gp} 
The proof is due to Lemma 1 by \citet{tianyu2019}. We present the proof for completeness. 
 \begin{lemman}[\ref{lem:point-scattering}] 
We say a partition $\mathcal{Q}$ is finer than a partition $\mathcal{Q}'$ if for any $q \in \mathcal{Q}$, there exists $q' \in \mathcal{Q}'$ such that $q \subset q' $.
For an arbitrary sequence of points $x_1, x_2, \cdots$ in a space $\mathcal{X}$, and a sequence of partitions $\mathcal{Q}_1, \mathcal{Q}_2, \cdots $ of the space $\mathcal{X}$ such that $\mathcal{Q}_{i+1}$ is finer than $\mathcal{Q}_{i}$ for all $i$, we have, for any $T$,
\begin{align}
    \sum_{t =  1 }^{T}  \frac{1}{ \widetilde{n}_{t } (Q_{t }  ) } &\le e | \mathcal{Q}_T |  \log \left( 1 + (e - 1) \frac{T }{  | \mathcal{Q}_T |  } \right) , \label{eq:point-scattering-gp-in-app} 
\end{align}
where $\widetilde{n}_{t}$ is defined in (\ref{eq:count}) (using points $x_1, x_2, \cdots$), and $|\mathcal{Q}_t|$ is the cardinality of partition $\mathcal{Q}_t$. 
\end{lemman}

 We use a constructive trick to derive (\ref{eq:point-scattering-gp-in-app}). 
 For each $T$, we construct a hypothetical noisy degenerate Gaussian process. We are not assuming our payoffs are drawn from these Gaussian processes. We only use these Gaussian processes as a proof tool.
 To construct these noisy degenerate Gaussian processes, we define the kernel functions $k_T: \mathcal{A} \times \mathcal{A} \rightarrow \mathbb{R}$ with respect to the partition $\mathcal{Q}_T$, 
 \begin{align}
 k_T (x, x^\prime) = \begin{cases} 1, \quad \text{if } x \text{ and }  x' \text{ are contained in the same element of } \mathcal{Q}_T, \\ 0, \quad \text{otherwise.}  \label{eq:tree-kernel}  \end{cases} 
 \end{align}
 The kernel $k_T$ is positive semi-definite as shown in Proposition \ref{prop}.
 \begin{proposition} \label{prop}
 The kernel defined in (\ref{eq:tree-kernel}) is positive semi-definite for any $T \ge 1$.
 \end{proposition}
 \begin{proof} 
 For any $x_1, \dots, x_n$ in where the kernel $k_T (\cdot, \cdot)$ is defined, the Gram matrix $ K = \begin{bmatrix}  k_T (x_i, x_j) \end{bmatrix}_{n \times n} $ can be written into block diagonal form where diagonal blocks are all-one matrices and off-diagonal blocks are all zeros with proper permutations of rows and columns. Thus without loss of generality, for any vector $\bm{v} = [v_1, v_2, \dots, v_n] \in \mathbb{R}^n$, $\bm{v}^\top K \bm{v} = \sum_{b = 1}^B \left( \sum_{j:i_j \text{ in block }b } v_{i_j} \right)^2 \ge 0$ where the first summation is taken over all diagonal blocks and $B$ is the total number of diagonal blocks in the Gram matrix.
 \end{proof}
 
 Pick any $t$ and $T$. For a sequence of points $x_1, x_2, \cdots, x_t$ up to time $t$, and the partition $\mathcal{Q}_T$, define 
 \begin{align}
 	 n_{T,t}^0 (x) := \sum_{t'=1}^t \mathbb{I}_{ \left[ x_{t'} \text{ and } x \text{ are in the same element of } \mathcal{Q}_T \right] }. 
 \end{align}
 In particular, going back to the definition in (\ref{eq:count}), we have 
 \begin{align}
 	n_{t,t-1}^0 (x) = n_{t} (x) \label{eq:connect-n0n}
 \end{align}
 for any $x$ and $t$. 
 
 
 Now, at any time $T$, let us consider the model $\tilde{y}(x) = g(x) + \mathfrak{e} $ where $g$ is drawn from a Gaussian process $g \sim \mathcal{GP} \left( 0 , k_T (\cdot, \cdot) \right)$ and $\mathfrak{e} \sim \mathcal{N} (0, s_T^2)$ is the noise. Suppose that the arms and hypothetical payoffs $(x_1, \tilde{y}_1, x_2, \tilde{y}_2, \dots, x_t, \tilde{y}_t)$  are observed from this Gaussian process. The posterior variance for this Gaussian process after the observations at $x_1, x_2, \dots, x_t$ is 
 \begin{align*} 
 \sigma^2_{T, t} (x) = k_T (x,x) - \bm{k}^T (K + s^2_T I)^{-1} \bm{k} 
 \end{align*} 
 where $\bm{k} = [ k_T (x, x_1), \dots, k_T (x, x_t) ]^\top$, $K = [ k_T (x_i, x_j) ]_{t \times t}$ and $I$ is the identity matrix. In other words,  $\sigma^2_{T,t} (x)$ is the posterior variance using points up to time $t$ with the kernel defined by the partition at time $T$.

 After some matrix manipulation, we know that 
 \begin{align*} 
 \sigma^2_{T,t} (x) = 1 - \bm{1} [ \bm{1} \bm{1}^\top + s^2_T I ]^{-1} \bm{1},
 \end{align*} 
 where $\bm{1} = [1,\cdots,1]_{1 \times {n^0_{T,t} (x)} }^\top$.
 By the Sherman-Morrison formula, $[ \bm{1} \bm{1}^\top + s^2_T I ]^{-1} = s^{-2}_T I - \frac{ s^{-4}_T \bm{1} \bm{1}^\top }{ 1 + s^{-2}_T n^0_{T,t} (x) }$. Thus the posterior variance is 
 \begin{align}
 	\sigma^2_{T,t} (x) = \frac{1}{1 + s^{-2}_T n^0_{T,t} (x)}.  \label{eq:posterior}
 \end{align}

 Following the arguments in \cite{srinivas2009gaussian}, we derive the following results. Since $\bm{x}_t$ is deterministic, $H (\tilde{\bm{y}}_t , \bm{x}_t ) = H ( \tilde{\bm{y}}_t )  $. Since, by definition of a Gaussian process, $\tilde{\bm{y}}_t$ follows a multivariate Gaussian distribution, 
 \begin{align} 
 H (\tilde{\bm{y}}_t ) = \frac{1}{2} \log \left[ (2 \pi e )^t \det \left( K + s_T^2 I \right) \right]  \label{eq:entropy-LHS} 
 \end{align} 
 where $K = \begin{bmatrix}  k_T (x_i, x_j) \end{bmatrix}_{t \times t}$. On the other hand, we can recursively compute $H ( \tilde{\bm{y}}_t)$ by
 \begin{align} 
 H ( \tilde{\bm{y}}_t) &= H( \tilde{y}_t | \tilde{ \bm{y} }_{t-1} ) + H( \tilde{ \bm{y} }_{t-1} ) \nonumber \\ 
 &= H( \tilde{y}_t | x_t, \tilde{ \bm{y} }_{t-1}, \bm{x}_{t-1} ) + H( \tilde{ \bm{y} }_{t-1} ) \nonumber \\ 
 &=  \frac{1}{2} \log \left( 2 \pi e \left( s_T^2 + \sigma_{T,t-1}^2 (x_t) \right) \right) + H( \tilde{ \bm{y} }_{t-1} ) \nonumber \\ 
 &= \frac{1}{2} \sum_{\tau =1}^t \log \left( 2 \pi e \left( s_T^2 + \sigma_{T, \tau -1}^2 (x_\tau ) \right) \right) \label{eq:entropy-RHS} 
 \end{align} 
 By (\ref{eq:entropy-LHS}) and (\ref{eq:entropy-RHS}), 
 \begin{align} 
 \sum_{\tau=1}^t \log \left( 1 + s^{-2} \sigma_{ T, \tau -1}^2 (x_\tau ) \right)  = \log \left[ \det \left( s^{-2} K + I \right) \right].  \label{eq:sum-to-log-det}
 \end{align}  
 For the block diagonal matrix $K$ of size $t \times t$, let $h_i$ denote the size of block $i$ and $B^\prime$ be the total number of diagonal blocks. Then we have
 \begin{align}
 \det \left( s^{-2} K + I \right) &= \prod_{i = 1}^{B^\prime } \det \left( s^{-2} \bm{1}_{h_i} \bm{1}_{h_i}^\top + I_{h_i \times h_i} \right) \nonumber \\
 &= \prod_{i = 1}^{B^\prime }  \left( 1 + s^{-2} h_i \right) \label{eq:deter-lemma} \\
 &\le \left( 1 + \frac{s^{-2} t}{| \mathcal{P}_T |} \right)^{| \mathcal{P}_T |} \label{eq:bound-on-det}
 \end{align}
 where (\ref{eq:deter-lemma}) is due to the matrix determinant lemma and the last inequality is that the geometric mean is no larger than the arithmetic mean and that $|\mathcal{P}_t| \ge B^\prime$. 
 Therefore, 
 \begin{align*} 
 \sum_{\tau=1}^t \log   \left( 1 + s^{-2} \sigma_{ T, \tau -1}^2 (x_\tau ) \right)  \le | \mathcal{P}_T | \log \left( 1 + \frac{s^{-2} t }{ | \mathcal{P}_T | } \right).
 \end{align*}

 Since the function $h(\lambda) = \frac{\lambda}{\log (1 + \lambda)}$ is increasing for non-negative $\lambda$, 
 $$\lambda \le \frac{s^{-2}_T }{ \log ( 1 + s^{-2}_T ) } \log (1 + \lambda)$$ 
 for $\lambda \in [ 0, s^{-2}_T ]$. 
 Since $\sigma_{T,t} (x) \in [0, 1]$ for all $x$, 
 \begin{align}
 	 \sigma_{T, t}^2 (x) \le \frac{1}{ \log (1 + s^{-2}_T ) } \log \left( 1 + s^{-2}_T \sigma_{T, t}^2 (x) \right) \label{eq:sigma-to-logsigma}
 \end{align}
 for $t, T = 0, 1, 2, \cdots$. 
 Since the partitions grow finer, for $T_1 \le T_2$, we have
 \begin{align}
 	n_{T_1, t} (x) \ge n_{T_2, t} (x), \quad \forall x. \label{eq:from-finer-partition}
 \end{align}
 
 This gives $\sigma_{T_1, t}^2 (x) \le \sigma_{T_2, t}^2 (x)$. 
 Suppose we query at points $x_{ 1 }, \cdots, x_T$ in the Gaussian process $\mathcal{GP} (0, k_T ( \cdot, \cdot ) )$. Then, 
 \begin{align} 
 \sum_{t =  1 }^{T}  \frac{1}{\widetilde{n}_{t} (x_t)}  &\le \sum_{t =  1 }^{T} \frac{1 + s^{-2}_T }{ 1 + s^{-2}_T {n}_{t} (x_t)}  \nonumber \\
 &= \sum_{t =  1 }^{T}  \frac{1 + s^{-2}_T }{ 1 + s^{-2}_T n^0_{t, t-1} (x_t)}  \label{eq:use-connection} \\
 &\le \sum_{t =  1 }^{T}  \frac{1 + s^{-2}_T }{ 1 + s^{-2}_T n^0_{T, t-1} (x_t)}  \label{eq:use-T2t} \\
 &\le  \left( 1 + s^{-2}_T \right) \sum_{t =  1 }^{T} \sigma^2_{T, t-1} (x_t)  \label{eq:use-posterior} \\
 &\le  \frac{ 1 + s^{-2}_T }{ \log ( 1 + s^{-2}_T ) } \sum_{t = 1 }^{T} \log \left( 1 + s^{-2}_T \sigma_{T, t-1}^2 (x_t) \right)  \label{eq:use-sigma-to-logsigma} \\ 
 &\le  \frac{ 1 + s^{-2}_T }{ \log ( 1 + s^{-2}_T ) } | \mathcal{Q}_T |  \log \left( 1 + s^{-2}_T \frac{T }{  | \mathcal{Q}_T |  } \right)  \label{eq:lemma4},
\end{align} 
where (\ref{eq:use-connection}) uses (\ref{eq:connect-n0n}), (\ref{eq:use-T2t}) uses (\ref{eq:from-finer-partition}), (\ref{eq:use-posterior}) uses (\ref{eq:posterior}), (\ref{eq:use-sigma-to-logsigma}) uses (\ref{eq:sigma-to-logsigma}), and (\ref{eq:lemma4}) uses (\ref{eq:sum-to-log-det}) and (\ref{eq:bound-on-det}). 
 
Finally, we optimize over $s_T$. Since $s_T^{-2} = e - 1$ minimizes $\frac{ 1 + s^{-2}_T }{ \log ( 1 + s^{-2}_T ) }$, (\ref{eq:lemma4}) gives
\begin{align*}
    \sum_{t =  1 }^{T} \frac{1}{ \widetilde{n}_{t} (x_t)} &\le e | \mathcal{P}_T | \log \left( 1 + ( e - 1 ) \frac{ T }{ | \mathcal{P}_T | } \right) . 
\end{align*}

\section{Additional Proof: Proof of Theorem \ref{thm:john-nirenberg}}
\label{app:jn}

In this part, we provide a proof to the John-Nirenberg inequality (Theorem \ref{thm:john-nirenberg}). Proofs to the John-Nirenberg inequality can be found in many textbooks on BMO functions or harmonic analysis \citep[e.g.,][]{stein1993harmonic}. Here, we present a proof by \citet{jnproof} for completeness. 


\begin{theoremn}[\ref{thm:john-nirenberg}](John-Nirenberg inequality) 
    Let $\mu$ be the Lebesgue measure.
    Let $f\in BMO \left(\mathbb{R}^{d}, \mu \right)$. Then there exists constants $C_1$ and $C_2$, such that, for any hypercube $Q \subset \mathbb{R}^{d}$ and any $\lambda > 0$, 
\begin{align*}
    &\mu \left( \left\{ x\in Q:\left|f\left(x\right)- \left<f \right>_{Q}\right|  \lambda\right\} \right) \le C_1 { \mu ( Q ) } \exp \left\{-\frac{ \lambda}{ C_2 \|f\| } \right\}. 
\end{align*}
\end{theoremn} 

\begin{proof}
The proof uses dyadic decomposition. 
By scaling, without loss of generality, we assume $\|f\| = 1$. Recall that $\mu$ is the Lebesgue measure. For a cube $Q \subset \mathbb{R}^d$, and $\omega > 0$, define
\begin{align} E\left(Q,\omega\right) & =
\left\{ 
x\in Q:
| f\left(x\right)-\left<f\right>_{Q} | > \omega
\right\} ,\\ 
\varphi\left( \omega \right) 
& = \sup_{Q}\frac{ \mu \left( E\left(Q,\omega \right) \right) }{ \mu ( Q ) } \label{eq:jn-def-phi}
\end{align}

We want to show that $\varphi\left(\omega\right) \lesssim e^{-\frac{\omega}{c}}$. First take $\omega>e>1$. Then
\begin{align*}
    \frac{1}{\mu (Q)} \int_{Q} \left|f- \left< f \right>_{Q}\right| d\mu \le \| f \|= 1 \le \omega
\end{align*}
for any Q. 
Subdivide Q dyadically and stop when 
\begin{align}
    \frac{1}{\mu(Q^\prime)} \int_{Q^{\prime}} \left|f- \left<f \right>_{Q^\prime}\right|>\omega.  \label{eq:jn-def-Qj}
\end{align}
Collect all such cubes ($Q^\prime$) to form a set $\mathcal{Q}=\left\{ Q_{j}\right\} _{j}. $ Note that the cubes in $\mathcal{Q}$ are disjoint. 
It could be $\mathcal{Q}=\varnothing$. 
Note that $\mathcal{Q}\subset\mathbb{D}_{Q}\setminus\left\{ Q\right\},$ 
where $\mathbb{D}_{Q}$ denotes the family of all dyadic cubes of $Q$.

Now we introduce the following Hardy–Littlewood type maximum $M_{Q}$, such that for a BMO function $g$, 
\begin{align}
    M_{Q}g (x) 
    :=\sup_{Q^{\prime}\in\mathbb{D}_{Q}, Q^\prime \ni x } \frac{1}{\mu (Q^\prime)} \int_{Q^\prime} g d\mu.
\end{align} 
Take $g=\left|f- \left<f \right>_{Q}\right|.$ 
Then by definition of $Q_{j}$, we have
\begin{align}
    \left\{ x \in Q:M_{Q} g (x)
    >\omega \right\} =\bigcup_{Q_{j}\in\mathcal{Q}}Q_{j}. 
\end{align}
For almost every $x\in E \left( Q,\omega \right)$, we have
\begin{align}
    \omega <\left|f-\left<f\right>_{Q}\right|=g\left(x\right)\le M_{Q}g \left(x\right). 
\end{align} 
So
\begin{align} 
E\left(Q,\omega\right) \subset \bigcup_{Q_{j}\in\mathcal{Q}}Q_{j}\quad\text{ almost everywhere}. \label{eq:jn-ae}
\end{align}

Let $\tilde{Q}_{j}$ be a parent cube of $Q_{j}$. Then since $ \mu (\tilde{Q}_j) = 2^d \mu (Q_j) $, \begin{align} \omega &< \int_{Q_{j}}
\left|f- \left< f \right>_{Q}\right| d\mu \\
&\le \frac{ \mu( \tilde{Q}_{j} ) }{ \mu( Q_{j} ) }\int_{\tilde{Q}_{j}}\left|f- \left<f \right>_{Q}\right| d\mu
\le 2^d \omega. \label{eq:jn-parent-cube}
\end{align}

Thus, for $Q_j \in \mathcal{Q}$, 
\begin{align} 
    \left|f\left(x\right)- \left< f \right>_{Q}\right| & \le \left|f\left(x\right)- \left< f \right>_{Q_{j}}\right|+\left| \left< f \right>_{Q_{j}}- \left< f \right>_{Q}\right|\\ 
    & \le \left|f\left(x\right)- \left< f \right> _{Q_{j}}\right|+\int_{Q_{j}}\left|f- \left< f \right>_{Q}\right| d\mu \\ 
    & \le\left|f\left(x\right)- \left< f \right>_{Q_{j}}\right| + 2^{d} \omega. 
\end{align}

Now, pick $ \zeta > 2^{d}\omega $. For $x\in E\left(Q, \zeta \right)$, we have, for $Q_j \in \mathcal{Q}$
\begin{align}
    \zeta < \left|f\left(x\right)- \left< f \right>_{Q}\right| \le
    \left|f\left(x\right)- \left< f \right>_{Q_{j}}\right| +2^{d}\omega . 
\end{align}

Hence for $x \in E(Q, \zeta)$, $\left|f\left(x\right) - \left<f \right>_{Q_{j}}\right|>\zeta -2^{d}\omega$ is necessary when $\left|f\left(x\right) - \left<f \right>_{Q}\right|> \omega.$

Since $\zeta >\omega$, by (\ref{eq:jn-ae}) we have
\begin{align} 
    \mu( E\left(Q,\zeta\right) ) 
    & =\mu (E\left(Q,\zeta \right) \cap E\left(Q, \omega \right) )
    \\ & \le\sum_{j} \mu ( E\left(Q, \zeta \right)\cap Q_{j} ) \label{eq:jn-use} \\ 
    & \le\sum_{j}\frac{ \mu \left( \left\{ x\in Q_{j}:\left|f\left(x\right)- \left< f\ \right>_{Q_{j}}\right|> \zeta -2^{d} \omega \right\} \right) }{ \mu ( Q_{j} ) } \mu( Q_{j} ) \label{eq:jn-use-necessary}\\ 
    & \le\varphi\left(\zeta-2^{d}\lambda\right)\sum_{j}\le \mu( Q_{j} ), \label{eq:jn-link-Ezeta-Qj}
\end{align}
where (\ref{eq:jn-use}) is due to disjointness of $Q_j$ and (\ref{eq:jn-ae}), and (\ref{eq:jn-use-necessary}) uses that, for $x \in E(Q, \zeta)$, 
$$
\left|f\left(x\right) - \left<f \right>_{Q}\right|> \omega
 \quad \Rightarrow \quad
\left|f\left(x\right) - \left<f \right>_{Q_{j}}\right|>\zeta -2^{d}\omega,
$$
as discussed above. 

Then we have
\begin{align} 
    \mu \left( E\left(Q,\zeta \right) \right)  & \le\varphi\left( \zeta -2^{d} \omega\right)\frac{1}{\omega }\sum_{j} \int_{Q_{j}}\left|f- \left< f \right>_{Q} \right| d\mu \label{eq:jn-use-def-Qj} \\ 
    & \le\frac{1}{\omega}  \varphi\left(\zeta-2^{d}\omega\right) \mu( Q ),  \label{eq:use-def-bmo-norm-1}
\end{align}
where we use (\ref{eq:jn-link-Ezeta-Qj}) and (\ref{eq:jn-def-Qj}) for (\ref{eq:jn-use-def-Qj}), and use the definition of BMO functions and $\|f\| = 1$ for (\ref{eq:use-def-bmo-norm-1}). 

Hence for $\zeta>2^{d}\omega$, we obtain
\begin{align}
    \frac{ \mu( E \left( Q, \zeta \right) ) }{\mu ( Q )} \le 
    \frac{1}{\omega}\varphi\left(\zeta-2^{d}\omega\right). 
\end{align}
By taking supremum over $Q$ on the left-hand-on of the above equation, we have
\begin{align}
    \varphi \left(\zeta \right)
    \le
    \frac{\varphi \left( \zeta -2^{d}\omega\right)}{\omega}. \label{eq:jn-de}
\end{align}

Put $\omega=e$. Note that $\varphi\left( \zeta \right) \le 1$ for all $\zeta > 0$ by Definition in (\ref{eq:jn-def-phi}). Then for $0< \zeta \le e\cdot2^{d}$, we have
\begin{align} 
    \varphi\left( \zeta \right) \le e\cdot e^{-\frac{\zeta}{2^{d}e}}.
    \label{eq: jn-lemma}
\end{align} 
The above statement is true by the proof of contradiction. Assume 
\begin{align}
   \varphi\left( \zeta \right) > e\cdot e^{-\frac{\zeta}{2^{d}e}} = e^{1 - \frac{\zeta}{2^{d}e}}
   \label{eq: contradiction-assumption}
\end{align}
Since for all $\zeta > 0$, $\varphi \left( \zeta \right) \le 1$, we have $1 - \frac{\zeta}{2^{d}e} < 0$ always true. This implies 
\begin{align}
    \zeta > e\cdot 2^{d}
\end{align}
This is to say if $\zeta > 0$, then $\zeta > e\cdot 2^d$. Hence (\ref{eq: contradiction-assumption}) implies the domain $\zeta \in (-\infty, 0]\cup (e\cdot2^d, +\infty)$. This shows (\ref{eq: jn-lemma}).


Next, note that 
\begin{align}
    \left(0,\infty\right)=(0,e\cdot2^{d}]\cup\left[\bigcup_{k=1}^{\infty}\left(e\cdot2^{d+k-1},e\cdot2^{d+k}\right]\right].
\end{align}
So for $e\cdot2^{d}< \zeta \le e\cdot2^{d+1}$, $\varphi\left(\zeta \right)\le e\cdot e^{-\frac{\zeta }{2^{d}e}}.$ Since we have, 
\begin{align}
    \varphi\left( \zeta \right)\le\frac{1}{e}\varphi\left(\zeta- e\cdot2^{d }\right),\quad e \cdot2^{d+1} < \zeta \le e \cdot2^{d+1}. 
\end{align}
We see that $\varphi\left(\zeta -e\cdot2^{d}\right)\le e\cdot e^{-\frac{\left(\zeta-2^{d}e\right)}{2^{d}e}}$ for $\zeta >e\cdot2^d.$ Hence, for $e\cdot2^{d}< \zeta <e\cdot2^{d+1},$ we have
  \[ \varphi\left(\zeta \right)\le e\cdot e^{-\frac{\zeta }{2^{d}e}}. \]
Iterate this procedure, and we obtain the desired claim, which is, $\forall \zeta > 0$, 
\begin{align}
    \frac{ \mu ( E\left(Q,\zeta \right) ) }{ \mu ( Q ) }
    \le 
    e\cdot e^{-\frac{\zeta}{2^{d}e}}\quad \text{for every cube }Q. 
\end{align}



\end{proof}

\section{Landscape of Test Functions in Section \ref{sec:exp}} \label{app:exp}

\begin{figure}[h!]
    \centering
    \includegraphics[scale = 0.5]{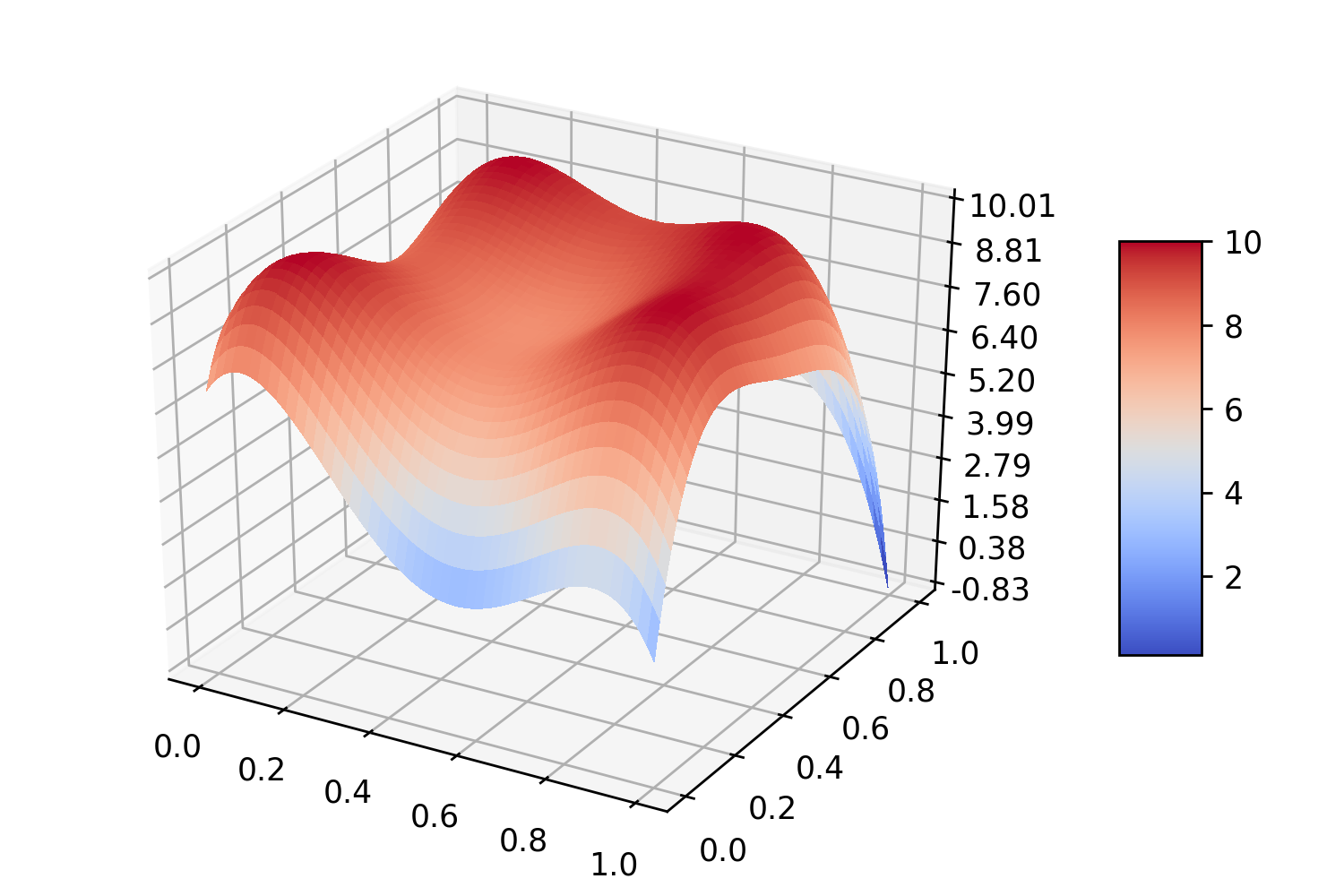} \includegraphics[scale = 0.5]{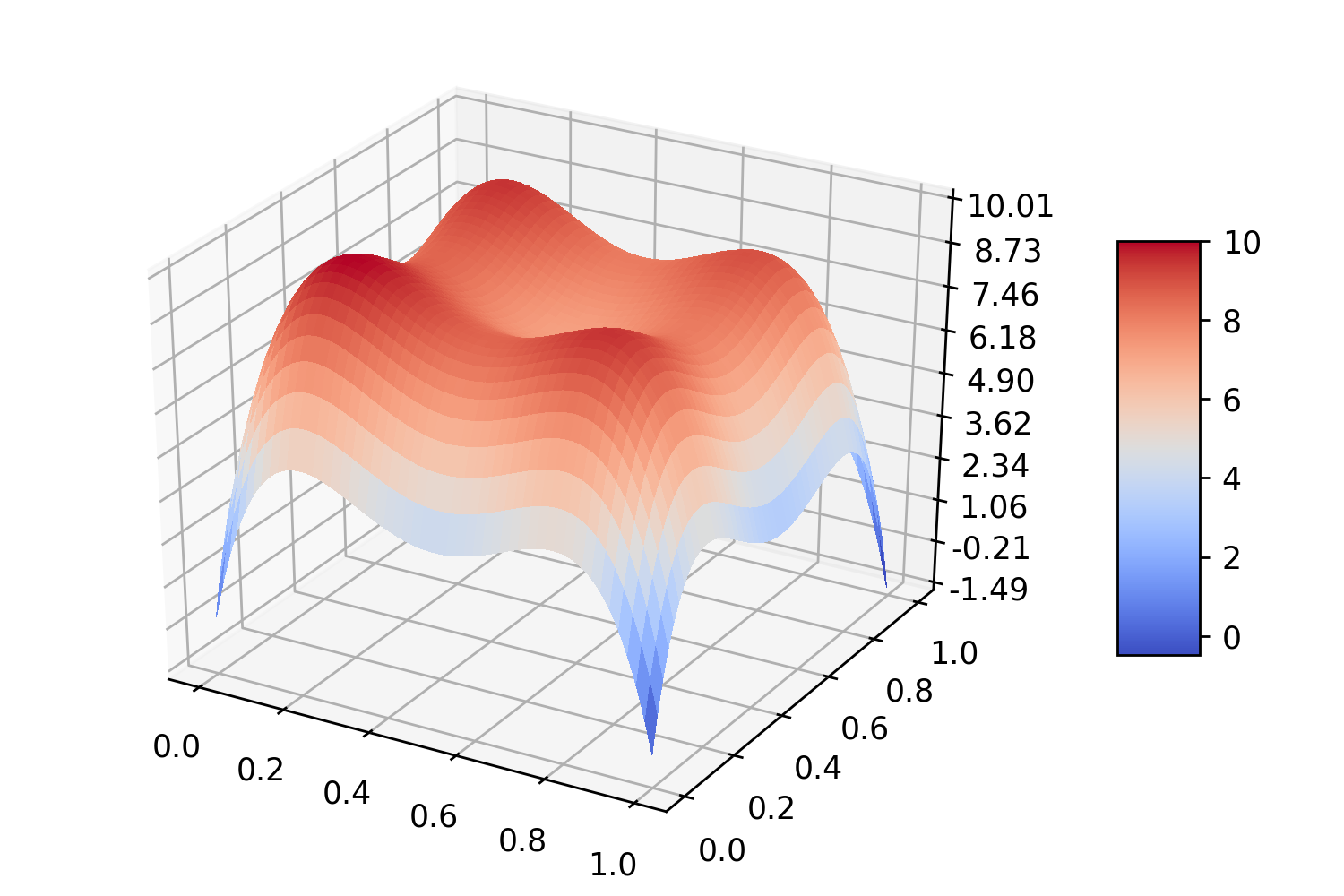}
    \caption{Landscapes of test functions used in Section \ref{sec:exp}. \textit{Left:} (Rescaled)  Himmelblau's function. \textit{Right:} (Rescaled) Styblinski-Tang function. \label{fig:landscape}}
\end{figure}

\end{document}